\documentclass[12pt,twoside]{article}

\usepackage{graphicx} 
\usepackage{subfigure} 

\graphicspath{{figures/}}

\usepackage[numbers,sort&compress]{natbib}

\usepackage{tikz}

\usepackage{algorithm}
\usepackage{algorithmic}

\usepackage{hyperref}

\def\tt{\theta}

\def\num_part{\#\pi}
\def\N_srg{N^{\num_part}}
\def\T_srg{T^{\num_part}}

\newenvironment{keywords}{\centerline{\bf\small
		Keywords}\begin{quote}\small}{\par\end{quote}\vskip 1ex}

\newenvironment{proof}[1][Proof]{\begin{trivlist}
		\item[\hskip \labelsep {\bfseries #1}]}{\end{trivlist}}
\newenvironment{definition}[1][Definition]{\begin{trivlist}
		\item[\hskip \labelsep {\bfseries #1}]}{\end{trivlist}}
\newenvironment{example}[1][Example]{\begin{trivlist}
		\item[\hskip \labelsep {\bfseries #1}]}{\end{trivlist}}
\newenvironment{remark}[1][Remark]{\begin{trivlist}
		\item[\hskip \labelsep {\bfseries #1}]}{\end{trivlist}}

\makeatletter
\newcommand{\figcaption}{\def\@captype{figure}\caption}
\newcommand{\tabcaption}{\def\@captype{table}\caption}
\makeatother

\newcommand{\qed}{\nobreak \ifvmode \relax \else
	\ifdim\lastskip<1.5em \hskip-\lastskip
	\hskip1.5em plus0em minus0.5em \fi \nobreak
	\vrule height0.75em width0.5em depth0.25em\fi}

\usepackage{amssymb}
\usepackage{amsmath}
\usepackage{helvet}  
\usepackage{courier}  
\usepackage{adjustbox}
\usepackage{xcolor}
\usepackage{amsmath}

\usepackage[utf8]{inputenc} 
\usepackage[T1]{fontenc}    
\usepackage{hyperref}       
\usepackage{url}            
\usepackage{booktabs}       
\usepackage{amsfonts}       
\usepackage{nicefrac}       
\usepackage{microtype}      

\usepackage{graphicx} 
\graphicspath{{figures/}}
\usepackage{wrapfig}
\usepackage{algorithm}
\usepackage{algorithmic}

\usepackage{multirow}

\allowdisplaybreaks[1]

\ifx\proof\undefined
\newenvironment{proof}{\par\noindent{\bf Proof\ }}{\hfill\BlackBox\\[2mm]}
\fi

\ifx\theorem\undefined
\newtheorem{theorem}{Theorem}
\fi
\ifx\example\undefined
\newtheorem{example}{Example}
\fi
\ifx\lemma\undefined
\newtheorem{lemma}[theorem]{Lemma}
\fi

\ifx\corollary\undefined
\newtheorem{corollary}[theorem]{Corollary}
\fi

\ifx\assumption\undefined
\newtheorem{assumption}{Assumption}
\fi

\ifx\definition\undefined
\newtheorem{definition}{Definition}
\fi

\ifx\proposition\undefined
\newtheorem{proposition}[theorem]{Proposition}
\fi

\ifx\remark\undefined
\newtheorem{remark}{Remark}
\fi

\ifx\conjecture\undefined
\newtheorem{conjecture}[theorem]{Conjecture}
\fi

\ifx\factoid\undefined
\newtheorem{factoid}[theorem]{Fact}
\fi

\ifx\axiom\undefined
\newtheorem{axiom}[theorem]{Axiom}
\fi

\newcommand{\RN}[1]{%
	\textup{\lowercase\expandafter{\it \romannumeral#1}}%
}

\def\KL{\textsf{KL}} 
\def\sig{\textsf{sigmoid}}

\input{Definitions}

\begin{document}
%
\title{Variance Reduction in Stochastic Particle-Optimization Sampling}
\author{
	{\bf Jianyi Zhang}, \\
	jianyi1zh@gmail.com \\
	School of Mathematical Sciences, Fudan University\\
	\and
	{\bf Yang Zhao} \\
	yzhao63@buffalo.edu \\
	SUNY at Buffalo \\
	\and
	{\bf Changyou Chen}, \\
	cchangyou@gmail.com \\
	SUNY at Buffalo\\
}

\maketitle

\begin{abstract} 
	Stochastic particle-optimization sampling (SPOS) is a recently-developed scalable Bayesian sampling framework that unifies stochastic gradient MCMC (SG-MCMC) and Stein variational gradient descent (SVGD) algorithms based on Wasserstein gradient flows. With a rigorous non-asymptotic convergence theory developed recently, SPOS avoids the particle-collapsing pitfall of SVGD. Nevertheless, variance reduction in SPOS has never been studied. In this paper, we bridge the gap by presenting several variance-reduction techniques for SPOS. Specifically, we propose three variants of variance-reduced SPOS, called SAGA particle-optimization sampling (SAGA-POS), SVRG particle-optimization sampling (SVRG-POS) and a variant of SVRG-POS which avoids full gradient computations, denoted as SVRG-POS$^+$. Importantly, we provide non-asymptotic convergence guarantees for these algorithms in terms of 2-Wasserstein metric and analyze their complexities. Remarkably, the results show our algorithms yield better convergence rates than existing variance-reduced variants of stochastic Langevin dynamics, even though more space is required to store the particles in training. Our theory well aligns with experimental results on both synthetic and real datasets.
\end{abstract} 


\newpage

\tableofcontents

\section{Introduction}

Sampling has been an effective tool for approximate Bayesian inference, which becomes increasingly important in modern machine learning. In the setting of big data, recent research has developed scalable Bayesian sampling algorithms such as stochastic gradient Markov Chain Monte Carlo (SG-MCMC) \cite{WellingT:ICML11} and Stein variational gradient descent (SVGD) \cite{liu2016stein}. These methods have facilitated important real-world applications and achieved impressive results, such as topic modeling \cite{GanCHCC:icml15,LiuZS:NIPS16}, matrix factorization \cite{ChenFG:ICML14,DingFBCSN:NIPS14,SBCG:ICML16}, differential privacy \cite{WangFS:icml15,LiCLC:arxiv17}, Bayesian optimization \cite{SpringenbergKFH:NIPS16} and deep neural networks \cite{PSGLD:AAAI16}. Generally speaking, these methods use gradient information of a target distribution to generate samples, leading to more effective algorithms compared to traditional sampling methods. Recently, \cite{ChenZWLC:UAI18} proposed a particle-optimization Bayesian sampling framework based on Wasserstein gradient flows, which unified SG-MCMC and SVGD in a new sampling framework called particle-optimization sampling (POS). Very recently, \cite{ZhangZC:18} discovered that SVGD endows some unintended pitfall, {\it i.e.} particles tend to collapse under some conditions. As a result, a remedy was proposed to inject random noise into SVGD update equations in the POS framework, leading to stochastic particle-optimization sampling (SPOS) algorithms \cite{ZhangZC:18}. Remarkably, for the first time, non-asymptotic convergence theory was developed for SPOS (SVGD-type algorithms) in \cite{ZhangZC:18}.

In another aspect, in order to deal with large-scale datasets, many gradient-based methods for optimization and sampling use stochastic gradients calculated on a mini-batch of a dataset for computational feasibility. Unfortunately, extra variance is introduced into the algorithms, which would potentially degrade their performance. Consequently, variance control has been an important and interesting work for research. Some efficient solutions such as SAGA \cite{SAGA:DA2014} and SVRG \cite{RTASGD:AG2013} were proposed to reduce variance in stochastic optimization. Subsequently, \cite{DubeyRPSX:nips16} introduced these techniques in SG-MCMC for Bayesian sampling, which also has achieved great success in practice.

Since SPOS has enjoyed the best of both worlds by combining SG-MCMC and SVGD, it will be of greater value to further reduce its gradient variance. While both algorithm and theory have been developed for SPOS, no work has been done to investigate its variance-reduction techniques. Compared with the research on SG-MCMC where variance reduction has been well explored by recent work such as \cite{DubeyRPSX:nips16,Nilardri:2018,DIFAN:2018}, it is much more challenging for SPOS to control the variance of stochastic gradients. This is because from a theoretical perspective, SPOS corresponds to nonlinear stochastic differential equations (SDE), where fewer existing mathematical tools can be applied for theoretical analysis. Furthermore, the fact that many particles are used in an algorithm makes it difficult to improve its performance by adding modifications to the way they interact with each other.

In this paper, we take the first attempt to study variance-reduction techniques in SPOS and develop corresponding convergence theory. We adopt recent ideas on variance reduction in SG-MCMC and stochastic-optimization algorithms, and propose three variance-reduced SPOS algorithms, denoted as SAGA particle-optimization sampling (SAGA-POS), SVRG particle-optimization sampling (SVRG-POS) and a variant of SVRG-POS without full-gradient computations, denoted as SVRG-POS$^+$. For all these variants, we prove rigorous theoretical results on their non-asymptotic convergence rates in terms of 2-Wasserstein metrics. Importantly, our theoretical results demonstrate significant improvements of convergence rates over standard SPOS. Remarkably, when comparing our convergence rates with those of variance-reduced stochastic gradient Langevin dynamics (SGLD), our theory indicates faster convergence rates of variance-reduced SPOS when the number of particles is large enough. Our theoretical results are verified by a number of experiments on both synthetic and real datasets.

\section{Preliminaries}

\subsection{Stochastic gradient MCMC}

In Bayesian sampling, one aims at sampling from a posterior distribution $p(\thetab|\Xb)\propto p(\Xb|\thetab)p(\thetab)$, where $\thetab\in\mathbb{R}^d$ represents the model parameter, and $\Xb\triangleq \{\xb_j\}_{j=1}^N$ is the dataset. Let $p(\thetab|\Xb)=(1/Z)\exp(-U(\thetab))$, where 
\vspace{-0.1cm}
{\begin{align*}
	U(\thetab)=-\log p(\Xb|\thetab)- \log p(\thetab)
	\triangleq -\sum_{j=1}^N\log p(\xb_i | \thetab) - \log p(\thetab)
	\end{align*}\par \vspace{-0.2cm}}
\noindent is referred to as the potential energy function, and $Z$ is the normalizing constant. We further define the full gradient $F$ and individual gradient $F_j$ used in our paper:
\vspace{-0.2cm}
{\begin{align*}
	F_j(\thetab)\triangleq & -\nabla_{\thetab}\log p(\xb_j | \thetab) - \frac{1}{N}\nabla_{\thetab}\log p(\thetab)=\frac{1}{N}\nabla_{\thetab}U(\thetab)\\
	F(\thetab)\triangleq  &   \nabla_{\thetab}U(\thetab)=\sum_{j=1}^{N}F_j(\thetab)
	\end{align*}
	We can define a stochastic differential equation, an instance of It\'{o} diffusion whose stationary distribution equals to the target posterior distribution $p(\thetab|\Xb)$. For example, consider the following 1st-order Langevin dynamic:
	\begin{align}\label{eq:itodif}
	\mathrm{d}\thetab_t = -\beta^{-1}F(\thetab_t)\mathrm{d}t + \sqrt{2\beta^{-1}}\mathrm{d}\mathcal{W}_t~,
	\end{align}
	where $t$ is the time index; $\mathcal{W}_t \in \mathbb{R}^{d}$ is $d$-dimensional Brownian motion, and $\beta$ a scaling factor. By the Fokker-Planck equation \cite{Kolmogoroff:MA31,Risken:FPE89}, the stationary distribution of \eqref{eq:itodif} equals to $p(\thetab|\Xb)$.
	
	SG-MCMC algorithms are discretized numerical approximations of It\'{o} diffusions \eqref{eq:itodif}.
	To make algorithms efficient in a big-data setting, the computationally-expensive term $F$ is replaced with its unbiased stochastic approximations with a random subset of the dataset in each interation, {\it e.g.} $F$ can be approximated by a stochastic gradient:
	{\small
		\begin{align*}
		\vspace{-0.3cm}
		G_k \triangleq \frac{N}{B} \sum_{j\in I_k}F_j(\theta_k) =-\nabla\log p(\theta_k)- \frac{N}{B}\sum_{j\in I_k} \nabla_{\theta_k}\log p(\xb_{j}|\theta_k)
		\end{align*}}
	\noindent where $I_k$ is a random subset of $\{1, 2, \cdots, N\}$ with size $B$. The above definition of $G_k$ reflects the fact that we only have information from $B \ll N$ data points in each iteration. This is the resource where the variance we try to reduce comes from.
	We should notice that $G_k$ is also used in standard SVGD and SPOS.
	As an example, SGLD is a numerical solution of \eqref{eq:itodif}, with an update equation:
	$\theta_{k+1} = \theta_{k} - \beta^{-1}G_kh + \sqrt{2\beta^{-1}h}\xib_{k}$, 
	where $h$ means the step size and $\xib_{k}\sim\mathcal{N}(\mathbf{0}, \Ib)$.

	\subsection{Stein variational gradient descent}
	
	Different from SG-MCMC, SVGD initializes a set of particles, which are iteratively updated to approximate a posterior distribution. Specifically, we consider a set of particles $\{\thetab^{(i)}\}_{i=1}^M$ drawn from some distribution $q$. SVGD tries to update these particles by doing gradient descent on the interactive particle system via
	\begin{align*}
	\thetab^{(i)} \leftarrow \thetab^{(i)} + h \phi(\thetab^{(i)}),~~\phi = \arg\max_{\phi\in \mathcal{F}} \{\dfrac{\partial}{\partial h} \KL(q_{[h\phi]}||p)\}
	\end{align*}
	where $\phi$ is a function perturbation direction chosen to minimize the KL divergence between the updated density $q_{[h\phi]}$ induced by the particles and the posterior $p(\thetab|\Xb)$.
	The standard SVGD algorithm considers $\mathcal{F}$ as the unit ball of a vector-valued reproducing kernel Hilbert space (RKHS) $\mathcal{H}$ associated with a kernel $\kappa(\thetab,\thetab^{\prime})$. In such a setting, \cite{liu2016stein} shows that
	\begin{align}
	\label{equ:close}
	\phi(\thetab) = \mathbb{E}_{\thetab^\prime\sim q}[\kappa(\thetab, \thetab^\prime) F(\thetab^\prime)
	+ \nabla_{\thetab^\prime} \kappa(\thetab, \thetab^\prime)].
	\end{align}
	When approximating the expectation $\mathbb{E}_{\thetab^\prime\sim q}[\cdot]$ with an empirical distribution formed by a set of particles $\{\thetab^{(i)}\}_{i=1}^M$ and adopting stochastic gradients $G_k^{(i)}\triangleq \frac{N}{B} \sum_{j\in I_k}F_j(\theta_k^{(i)})$, we arrive at the following update for the particles: 
	{\small\begin{align}  \label{eq:svgd_update}
		\theta_{k+1}^{(i)} = \theta_{k}^{(i)} +\dfrac{h}{M} \sum_{q=1}^M \left[ \kappa(\theta_{\k}^{(q)}, \theta_{k}^{(i)}) G_k^{(i)}+ \nabla_{\theta_{k}^{(q)}} \kappa(\theta_{k}^{(q)}, \theta_{k}^{(i)}) \right]
		\end{align}}
	SVGD then applies \eqref{eq:svgd_update} repeatedly for all the particles.
	
	\subsection{Stochastic particle-optimization sampling}\label{SPOS}
	In this paper, we focus on RBF kernel $\kappa(\thetab,\thetab^{\prime})=\exp(-\frac{\| \thetab-\thetab^{\prime} \|^2}{2\eta^2})$ due to its wide use in both theoretical analysis and practical applications. Hence, we can use a function $K(\thetab)=\exp(-\frac{\| \thetab \|^2}{2\eta^2})$ to denote the kernel $\kappa(\thetab,\thetab^{\prime})$. According to the work of \cite{ChenZWLC:UAI18,ZhangZC:18}, the stationary distribution of the $\rho_t$ in the following partial differential equation equals to $p(\thetab|\Xb)$. 
	\begin{align}\label{accept}
	\partial_t \rho_t =& \nabla_{\thetab}\cdot (\rho_t\beta^{-1}F(\thetab)+\rho_tE_{Y\sim\rho_t}K(\thetab-Y)F(Y)-\rho_{t}(\nabla K*\rho_{t})+\beta^{-1}\nabla_{\thetab}\rho_{t})~.
	\end{align}
	When approximating the $\rho_t$ in Eq.(\ref{accept}) with an empirical distribution formed by a set of particles $\{\thetab^{(i)}\}_{i=1}^M$, \cite{ZhangZC:18} derive the following diffusion process characterizing the SPOS algorithm. 
	{\small\begin{align} \label{eq:particle}
		&\mathrm{d}\thetab_{t}^{(i)} =-\beta^{-1}F(\thetab_t^{(i)})\mathrm{d}t - \frac{1}{M}\sum_{q=1}^{M}K(\thetab_{t}^{(i)} - \thetab_{t}^{(q)})F(\thetab_{t}^{(q)})\mathrm{d}t \nonumber \\
		&+\frac{1}{M}\sum_{q=1}^{M}\nabla K(\thetab_{t}^{(i)} - \thetab_{t}^{(q)})\mathrm{d}t+ \sqrt{2\beta^{-1}}\mathrm{d}\mathcal{W}_t^{(i)}~~\forall i~
		\end{align}}
	It is worth noting that if we set the initial distribution of all the particles $\thetab_0^{(i)}$ to be the same, the system of these M particles is exchangeable. So the distributions of all the $\thetab_ t^{(i)}$ are identical and can be denoted as $\rho_t$. When solving the above diffusion process with a numerical method and adopting stochastic gradients $G_k^{(i)}$, one arrives at the SPOS algorithm of \cite{ZhangZC:18} with the following update equation:
	\begin{align} \label{eq:particle_num}
	&{\theta}_{k+1}^{(i)} = {\theta}_{k}^{(i)} -h\beta^{-1}G_k^{(i)}  - \frac{h}{M}\sum_{j=1}^{M}K(\theta_{k}^{(i)} - \theta_{k}^{(j)})G_k^{(j)} \nonumber \\
	& +\frac{h}{M}\sum_{j=1}^{M}\nabla K({\theta}_{k }^{(i)} - {\theta}_{k }^{(j)}) + \sqrt{2\beta^{-1}h}\xi_{k}^{(i)} 
	\end{align}
	where $\xi_{k}^{(i)}\sim\mathcal{N}(\mathbf{0}, \Ib)$. And SPOS will apply update \eqref{eq:particle_num} repeatedly for all the particles ${\theta}_{k}^{(i)}$. Detailed theoretical results for SPOS are reviewed in the Supplementary Material (SM).
	
\section{Variance Reduction in SPOS}\label{sec:alg}
In standard SPOS, each particle is updated by adopting $G_k^{(i)} \triangleq \frac{N}{B}\sum_{j\in I_k}F_j(\thetab_k^{(i)})$. Due to the fact that one can only access $B \ll N$ data points in each update, the increased variance of the ``noisy gradient'' $G_k^{(i)}$ would cause a slower convergence rate. A simple way to alleviate this is to increase $B$ by using larger minibatches. Unfortunately, this would bring more computational costs, an undesired side effect.
Thus more effective variance-reduction methods are needed for SPOS. Inspired by recent work on variance reduction in SGLD, {\it e.g.}, \cite{DubeyRPSX:nips16,Nilardri:2018,DIFAN:2018}, we develop three different variance-reduction algorithms for SPOS based on SAGA \cite{SAGA:DA2014} and SVRG \cite{RTASGD:AG2013} in stochastic optimization.
\subsection{SAGA-POS}
SAGA-POS generalizes the idea of SAGA \cite{SAGA:DA2014} to an interactive particle-optimization system. For each particle $\theta_k^{(i)}$, we use $\{g_{k,j}^{(i)}\}_{j=1}^N$ as an approximation for each individual gradient $F_j(\theta_k^{(i)})$. 
An unbiased estimate of the full gradient $F(\theta_k^{(i)})$ is calculated as:
\begin{align}\label{G}
G_k^{(i)}=\sum\limits_{j=1}^{N}g_{k,j}^{(i)}+\frac{N}{B}\sum\limits_{j\in{I_k}}(F_j(\theta_k^{(i)})-g_{k,j}^{(i)}),~\forall i~
\end{align}
In each iteration, $\{g_{k,j}^{(i)}\}_{j=1}^N$ will be partially updated under the following rule: $g_{k+1,j}^{(i)}=F_j(\theta_k^{(i)})$ if $j \in I_k $, and $g_{k+1,j}^{(i)}=g_{k,j}^{(i)}$ otherwise. 
The algorithm is described in Algorithm \ref{algo:algo1}.

Compared with standard SPOS, SAGA-POS also enjoys highly computational efficiency, as it does not require calculation of each $F_j(\theta_k^{(i)})$ to get the full gradient $F(\theta_k^{(i)})$ in each iteration. Hence, the computational time of SAGA-POS is almost the same as that of POS. However, our analysis in Section~\ref{analysis} shows that SAGA-POS endows a better convergence rate.

From another aspect, SAGA-POS has the same drawback of SAGA-based algorithms, which requires memory scaling at a rate of $\mathcal{O}(MNd)$ in the worst case. For each particle $\theta_k^{(i)}$, one needs to store N gradient approximations $\{g_{k,j}^{(i)}\}_{j=1}^N$. Fortunately, as pointed out by \cite{DubeyRPSX:nips16,Nilardri:2018}, in some applications, the memory cost scales only as $\mathcal{O}(N)$ for SAGA-LD, which corresponds to $\mathcal{O}(MN)$ for SAGA-POS as $M$ particles are used.

\begin{algorithm}\label{algo:algo1}
	\caption{SAGA-POS}
	{\bf Input:} A set of initial particles $\{\theta_0^{(i)}\}_{i=1}^{M}$, each $\theta_0^{(i)}\in{\mathbb{R}^d}$, step size $h_k$, batch size $B$.\\
	Initialize $\{g_{0,j}^{(i)}\}_{j=1}^{N}=\{F_j(\theta_0^{(i)})\}_{j=1}^{N}$ for all $i\in{\{1,...,M\}}$;
	\begin{algorithmic}[1]
		\FOR{iteration $k$= 0,1,...,T}
		\STATE Uniformly sample $I_k$ from $\{1,2,...,N\}$ randomly with replacement such that $|I_k|=B$; \\
		\STATE Sample $\xi_{k}^{(i)} \sim N($0$,I_{d \times d}),~\forall i~$;\\
		\STATE Update $G_k^{(i)}\leftarrow \sum\limits_{j=1}^{N}g_{k,j}^{(i)}+\frac{N}{B}\sum\limits_{j\in{I_k}}(F_j(\theta_k^{(i)})-g_{k,j}^{(i)}),~\forall i~$;\\
		\STATE Update each $\theta_k^{(i)}$ with Eq.(\ref{eq:particle_num});\\
		\STATE Update $\{g_{k,j}^{(i)}\}_{j=1}^N,~~\forall i~$: if $j \in{I_k}$, set $g_{k+1,j}^{(i)}\leftarrow F_j(\theta_k^{(i)})$; else, set $g_{k+1,j}^{(i)}\leftarrow g_{k,j}^{(i)}$
		\ENDFOR
	\end{algorithmic}
	{\bf Output:}{$\{\theta_T^{(i)}\}_{i=1}^M$}
\end{algorithm}

\begin{remark}\label{remark1}
	When compared with SAGA-LD, it is worth noting that $M$ particles are used in both SPOS and SAGA-POS. This makes the memory complexity $M$ times worse than SAGA-LD in training, thus SAGA-POS does not seem to bring any advantages over SAGA-LD. However, this intuition is not correct. As indicated by our theoretical results in Section~\ref{analysis}, when the number of particles $M$ is large enough, the convergence rates of our algorithms are actually better than those of variance-reduced SGLD counterparts. 
\end{remark}

\subsection{SVRG-POS}\vspace{-0.1cm}
Under limited memory, we propose SVRG-POS, which is based on the SVRG method of \cite{RTASGD:AG2013}. For each particle $\theta_k^{(i)}$, ones needs to store a stale parameter $\widetilde{\theta}^{(i)}$, and update it occasionally for every $\tau$ iterations. At each update, we need to further conduct a global evaluation of full gradients at $\widetilde{\theta}^{(i)}$,  {\it i.e.}, $\widetilde{G}^{(i)}\leftarrow F(\theta_k^{(i)})=F(\widetilde{\theta}^{(i)})$. 
An unbiased gradient estimate is then calculated by leveraging both $\widetilde{G}^{(i)}$ and $\widetilde{\theta}^{(i)}$ as:
\begin{align}
G_k^{(i)}\leftarrow \widetilde{G}^{(i)}+\frac{N}{B}\sum\limits_{j\in{I_k}}[F_j(\theta_k^{(i)})-F_j(\widetilde{\theta}^{(i)})]
\end{align}\par \vspace{-0.2cm}

\noindent The algorithm is depicted in Algorithm~\ref{algo:algo2}, where one only needs to store $\widetilde{\theta}^{(i)}$ and $\widetilde{G}^{(i)}$, instead of gradient estimates of all the individual $F_j$. Hence the memory cost scales as $\mathcal{O}(Md)$, almost the same as that of standard SPOS.

We note that although SVRG-POS alleviates the storage requirement of SAGA-POS remarkably, it also endows downside that full gradients, $F(\widetilde{\theta}^{(i)})=\sum_{j=1}^N F(\widetilde{\theta}^{(i)})$, are needed to be re-computed every $\tau$ iterations, leading to high computational cost in a big-data scenario.
\begin{algorithm}\label{algo:algo2}
	\caption{SVRG-POS}
	{\bf Input:} A set of initial particles $\{\theta_0^{(i)}\}_{i=1}^{M}$, each $\theta_0^{(i)}\in{\mathbb{R}^d}$, step size $h$, epoch length $\tau$, batch size $B$.\\
	Initialize $\{\widetilde{\theta}^{(i)}\}\leftarrow\{\theta_0^{(i)}\},\widetilde{G}^{(i)}\leftarrow F(\theta_0^{(i)}),~\forall i~$;
	\begin{algorithmic}[1]
		\FOR{iteration $k$= 0,1,...,T}
		\IF{k mod $\tau$ =0}
		\STATE{\textbf{Option \uppercase\expandafter{\romannumeral1}} $\RN{1}) $Sample $l\sim unif(0,1,..,\tau-1)$\\
			$\RN{2}) $Update $\widetilde{\theta}^{(i)} \leftarrow \theta_{k-l}^{(i)}$\\
			Update $\theta_{k}^{(i)} \leftarrow \widetilde{\theta}^{(i)},~\forall i~$\\
			$\RN{3}) $Update $\widetilde{G}^{(i)}\leftarrow F(\theta_k^{(i)}),~\forall i~$
			\\
			\STATE\textbf{Option \uppercase\expandafter{\romannumeral2}} $\RN{1})$ Update $\widetilde{\theta}^{(i)} \leftarrow \theta_k^{(i)}$\\
			$\RN{2})$Update $\widetilde{G}^{(i)}\leftarrow F(\theta_k^{(i)}),~\forall i~$\\
		}\ENDIF
		\STATE Uniformly sample $I_k$ from $\{1,2,...,N\}$ randomly with replacement such that $|I_k|=B$; \\
		\STATE Sample $\xi_{k}^{(i)} \sim N($0$,I_{d \times d}),~\forall i~$;\\
		\STATE Update $G_k^{(i)}\leftarrow \widetilde{G}^{(i)}+\frac{N}{B}\sum\limits_{j\in{I_k}}[F_j(\theta_k^{(i)})-F_j(\widetilde{\theta}^{(i)})],~\forall i~$;\\
		\STATE Update each $\theta_k^{(i)}$ with  Eq.(\ref{eq:particle_num})
		\ENDFOR
	\end{algorithmic}
	{\bf Output:}{$\{\theta_T^{(i)}\}_{i=1}^M$}
\end{algorithm}
\begin{remark}\label{remark2}
	$\RN{1})$ Similar to SAGA-POS, according to our theory in Section~\ref{analysis}, SVRG-POS enjoys a faster convergence rate than SVRD-LD -- its SGLD counterpart, although $M$ times more space are required for the particles. This provides a trade-off between convergence rates and space complexity. 
	$\RN{2})$ Previous work has shown that SAGA typically outperforms SVRG \cite{DubeyRPSX:nips16,Nilardri:2018} in terms of convergence speed. The conclusion applies to our case, which will be verified both by theoretical analysis in Section~\ref{analysis} and experiments in Section~\ref{sec:exp}.
\end{remark}

\subsection{SVRG-POS$^+$}

The need of full gradient computation in SVRG-POS motives the development of SVRG-POS$^+$. Our algorithm is also inspired by the recent work of SVRG-LD$^{+}$ on reducing the computational cost in SVRG-LD \cite{DIFAN:2018}. 
The main idea in SVRG-POS$^+$ is to replace the full gradient computation every $\tau$ iterations with a subsampled gradient, {\it i.e.}, to uniformly sample $|J_k| = b$ data points where $J_k$ are random samples from $\{1,2,...,N\}$ with replacement. Given the sub-sampled data, $\widetilde{\theta}^{(i)}$ and $\widetilde{G}^{(i)}$ are updated as:
$\widetilde{\theta}^{(i)}=\theta_k^{(i)},~~
\widetilde{G}^{(i)}=\frac{N}{b}\sum_{j\in J_k}F_j(\theta_k^{(i)})$. 
The full algorithm is shown in Algorithm~\ref{algo:algo3}.

\begin{algorithm}\label{algo:algo3}
	\caption{SVRG-POS$^+$}
	{\bf Input : } A set of initial particles $\{\theta_0^{(i)}\}_{i=1}^{M}$, each $\theta_0^{(i)}\in{\mathbb{R}^d}$, step size $h$, epoch length $\tau$, batch size $B$.\\
	Initialize $\{\widetilde{\theta}^{(i)}\}\leftarrow\{\theta_0^{(i)}\},\widetilde{G}^{(i)}\leftarrow F(\theta_0^{(i)}),~\forall i~$;
	\begin{algorithmic}[1]
		\FOR{iteration $k$= 0,1,...,T}
		\IF{k mod $\tau$ =0}
		\STATE$\RN{1})$ Uniformly sample $J_k$ from $\{1,2,...,N\}$ with replacement such that $|J_k|=b$;\\
		$\RN{2})$ Update $\widetilde{\theta}^{(i)} \leftarrow \theta_k^{(i)}$\;
		$\widetilde{G}^{(i)}\leftarrow \frac{N}{b}\sum_{j\in J_k}F_j(\theta_k^{(i)}),~\forall i~$;
		\ENDIF
		\STATE Uniformly sample $I_k$ from $\{1,2,...,N\}$ with replacement such that $|I_k|=B$; \\
		\STATE Sample $\xi_{k}^{(i)} \sim N($0$,I_{d \times d}),~\forall i~$;\\
		\STATE Update $G_k^{(i)}\leftarrow \widetilde{G}^{(i)}+\frac{N}{B}\sum\limits_{j\in{I_k}}[F_j(\theta_k^{(i)})-F_j(\widetilde{\theta}^{(i)})],~\forall i~$;\\
		\STATE Update each $\theta_k^{(i)}$ with  Eq.(\ref{eq:particle_num})
		\ENDFOR
	\end{algorithmic}
	{\bf Output:}{$\{\theta_T^{(i)}\}_{i=1}^M$}
\end{algorithm}

\section{Convergence Analysis}\label{analysis}
In this section, we prove non-asymptotic convergence rates for the SAGA-POS, SVRG-POS and SVRG-POS$^+$ algorithms under the 2-Wasserstein metric, defined as
\begin{align*}
\mathcal{W}_2(\mu,\nu)=\left(\inf\limits_{\zeta \in \Gamma(\mu,\nu)}\int _{\mathbb{R}^d\times\mathbb{R}^d} \|X_{\mu}-X_{\nu}\|^2 d\zeta(X_{\mu},X_{\nu}) \right)^{\frac{1}{2}}
\end{align*}
where $\Gamma(\mu,\nu)$ is the set of joint distributions on $\mathbb{R}^d\times\mathbb{R}^d$ with marginal distribution $\mu$ and $\nu$. 
Let $\mu^*$ denote our target distribution, and $\mu_T$ the distribution of $\frac{1}{M}\sum_{i=1}^{M}\theta_T^{(i)}$ derived via \eqref{eq:particle} after $T$ iterations. Our analysis aims at bounding $\mathcal{W}_2(\mu_T,\mu^{*})$. We first introduce our assumptions.

\begin{assumption}\label{ass:ass1}
	$F$ and $K$ satisfy the following conditions:
	\begin{itemize}		
		\item There exist two positive constants $m_F$ and $m_W$, such that $\langle F(\thetab)-F(\thetab^\prime), \thetab-\thetab^\prime \rangle \geq m_F\|\thetab-\thetab^\prime\|^2$ and $\langle \nabla K(\thetab)-\nabla K(\thetab^\prime), \thetab-\thetab^\prime \rangle \leq -m_K\|\thetab-\thetab^\prime\|^2$.
		\item $F$ is bounded and $L_F$-Lipschitz continuous with $L_F$ {\it i.e.} $\|F(\thetab)\| \leq H_F  $ and $\|F(\thetab) - F(\thetab^{\prime})\| \leq L_F\|\thetab - \thetab^{\prime}\|$; $K$ is $L_K$-Lipschitz continuous for some $L_{K} \geq 0$ and bounded by some constant $H_K >0 $.
		\item $K$ is an even function, {\it i.e.}, $K(-\thetab) = K(\thetab)$.
	\end{itemize}
\end{assumption}

\begin{assumption}\label{ass:ass2}
	There exists a constant $D_F > 0$ such that $\|\nabla F(\thetab)-\nabla F(\thetab^{\prime})\|\leq D_F\|\thetab-\thetab^{\prime}\|$.
\end{assumption}
\begin{assumption}\label{ass:ass3}
	There exits a constant $\sigma$ such that for all $j\in \{1,2,...,N\}$,
	\vspace{-0.3cm}
	\begin{align*}
	\mathbb{E}[\|F_j(\thetab)-\frac{1}{N}\sum \limits_{j=1}^{N}F_j(\thetab)\|^2]\leq d \sigma^2/N^2
	\end{align*}
	\vspace{-0.3cm}
\end{assumption}
\begin{remark}
	$\RN{1}) $ Assumption \ref{ass:ass1} is adopted from \cite{ZhangZC:18} which analyzes the convergence property of SPOS. The first bullet of Assumption \ref{ass:ass1} suggests $U(\cdot)$ is a strongly convex function, which is the general assumption in analyzing SGLD \cite{Dalalyan:2017,Durmus:2016sgld} and its variance-reduced variants \cite{DIFAN:2018,Nilardri:2018}. It is worth noting that although some work has been done to investigate the non-convex case, it still has significant value to analysis the convex case, which are more instructive and meaningful to address the practical issues \cite{Dalalyan:2017,Durmus:2016sgld,DIFAN:2018,Nilardri:2018}. $\RN{2})$ All of the $m_f$, $L_F$, $H_F$ and $D_F$ could scale linearly with $N$. $\RN{3})$ $K( \thetab )=\exp(-\frac{\| \thetab \|^2}{2\eta^2})$ can satisfy the above assumptions by setting the bandwidth large enough, since we mainly focus on some bounded space in practice. Consequently, $\nabla K$ can also be $L_{\nabla K}$-Lipschitz continuous and bounded by $H_{\nabla K}$; K can also be Hessian Lipschitz with some positive constant $D_{\nabla^2K}$
	
\end{remark}
For the sake of clarity, we define some constants which will be used in our theorems.
\begin{align*}
&C_1=\frac{H_{\nabla K}+H_F}{\sqrt{2}(\beta^{-1}-3H_FL_K-2L_F)}\\
&C_2=\sqrt{2(\beta^{-1}L_F+2L_K H_F+H_K L_F+L_{\nabla K})^2+2}\\
&C_3=\beta^{-1}m_F-2L_F-3H_FL_K\\
&C_4=\beta^{-1}D_F+4D_{\nabla^2K}+4H_FL_{\nabla K}+2L_FH_{\nabla K}\nonumber \\
&~~~~~+2H_FL_{K}+L_FH_{K}\\
&C_5=2\beta^{-1}\sigma^2+2H_K^2\sigma^2
\end{align*}
Now we present convergence analysis for our algorithms, where $\alpha$ is some positive constant independent of T. 

\begin{theorem} \label{thm:thm1}
	Let  $\mu_T$ denote the distribution of the particles after $T$ iterations with SAGA-POS. Under Assumption \ref{ass:ass1} and \ref{ass:ass2}, let the step size $h < \frac{B}{8C_2N}$ and the batch size $B \geq 9$, the convergence rate of SAGA-POS is bounded as
	\begin{align}
	\mathcal{W}_2(\mu_T,&\mu^*)\leq\frac{C_1}{\sqrt{M}}+5 \exp(-\frac{C_3 h}{4}T) \mathcal{W}_2(\mu_0,\mu^*) \nonumber \\
	&+\frac{2hC_4dM^{1/2-\alpha}}{C_3}+\frac{2h{C_2}^\frac{3}{2}\sqrt{d}}{C_3M^{\alpha}}+\frac{24hC_2\sqrt{dN}}{M^{\alpha}\sqrt{C_3}B}
	\end{align}
\end{theorem}

\begin{theorem} \label{thm:thm2}
	Let $\mu_T$ denote the distribution of the particles after $T$ iterations with SVRG-POS in Algorithm~\ref{algo:algo2}. 
	Under Assumption  \ref{ass:ass1} and \ref{ass:ass2}, if we choose Option $\uppercase\expandafter{\romannumeral1}$ and set the step size $h < \frac{1}{8C_2}$, the batch size $B \geq 2$ and the epoch length $\tau = \frac{4}{hC_3(1-2hC_2(1+2/B)) }$, the convergence rate of SVRG-POS is bounded for all T, which mod $\tau$ = 0, as 
	{\small\begin{align}
		\mathcal{W}_2(\mu_T,\mu^*)&\leq \frac{C_1}{\sqrt{M}}+ \exp(-\frac{C_3h}{56}T)\frac{\sqrt{C_2}}{\sqrt{C_3}} \mathcal{W}_2(\mu_0,\mu^*)\nonumber \\
		&+\frac{2hC_4dM^{1/2-\alpha}}{C_3}+\frac{2h{C_2}^\frac{3}{2}\sqrt{d}}{C_3M^{\alpha}}+\frac{64{C_2}^{\frac{3}{2}}\sqrt{hd}}{M^{\alpha}\sqrt{B}C_3}
		\end{align}}
	If we choose Option $\uppercase\expandafter{\romannumeral2}$ and set the step size $h < \frac{\sqrt{B}}{4\tau C_2}$, the convergence rate of SVRG-POS is bounded for all $T$ as 
	\begin{align}
	\mathcal{W}_2(\mu_T,\mu^*)&\leq \frac{C_1}{\sqrt{M}}+ \exp(-\frac{C_3h}{4}T)\mathcal{W}_2(\mu_0,\mu^*)\nonumber \\+
	&\frac{\sqrt{2}hC_4dM^{1/2-\alpha}}{C_3}+\frac{5h{C_2}^\frac{3}{2}\sqrt{d}}{C_3M^{\alpha}}+\frac{9hC_2\tau\sqrt{d}}{M^{\alpha}\sqrt{BC_3}}
	\end{align}
	
\end{theorem}

\begin{theorem} \label{thm:thm3}
	Let $\mu_T$ denote the distribution of particles after $T$ iterations with SVRG-POS$^{+}$. Under Assumption \ref{ass:ass1}, \ref{ass:ass2} and \ref{ass:ass3}, if we set the step size $h \leq \min\{(\frac{BC_3 }{24{C_2}^4\tau^2})^{\frac{1}{3}},\frac{1}{6\tau({C_5}^2/b+{C_2})}\}$, then the convergence rate of SVRG-POS$^+$ is bounded for all T as 
	\begin{align}
	\mathcal{W}_2(\mu_T,&\mu^*)\leq \frac{C_1}{\sqrt{M}}+(1-hC_3/4)^T\mathcal{W}_2(\mu_0,\mu^*)
	\nonumber \\
	&+\frac{3C_5 d^{1/2}}{M^{\alpha}C_3 b^{1/2}}\boldsymbol{1}(b\leq N)+\frac{2h(C_4dM^{1/2-\alpha})}{C_3}+\frac{2hC_2^{3/2}d^{1/2}}{C_3M^{\alpha}}\nonumber \\
	&+\frac{4hC_2(\tau d)^{1/2}\wedge 3h^{1/2}d^{1/2}C_5}{M^{\alpha}\sqrt{BC_3}}
	\end{align}
\end{theorem}
Since the complexity has been discussed in the Section~\ref{sec:alg}, we mainly focus on discussing the convergence rates here. Due to space limit, we move the comparison between convergence rates of the standard SPOS and its variance-reduced counterparts such as SAGA-POS into the SM. Specifically, 
adopting the standard framework of comparing different variance-reduction techniques in SGLD \cite{DubeyRPSX:nips16,Nilardri:2018,DIFAN:2018}, we focus on the scenario where $m_f$, $L_F$, $H_F$ and $D_F$ all scale linearly with $N$ with $N \gg d$. In this case, the dominating term in Theorem~\ref{thm:thm1} for SAGA-POS is the last term, $\mathcal{O}(\frac{hC_2\sqrt{d}}{M^{\alpha}B})\approx \mathcal{O}(\frac{hN\sqrt{d}}{M^{\alpha}B})$. Thus to achieve an accuracy of $\varepsilon$, we would need the stepsize $h_{ag}=\mathcal{O}(\frac{\varepsilon M^{\alpha}B}{N\sqrt{d}})$. For SVRG-POS, the dominating term in Theorem~\ref{thm:thm2} is  $\mathcal{O}(\frac{\sqrt{hNd}}{M^{\alpha}\sqrt{B}})$ for Option \uppercase\expandafter{\romannumeral1} and $\mathcal{O}(\frac{\tau hN\sqrt{d}}{M^{\alpha}\sqrt{B}})$ for Option \uppercase\expandafter{\romannumeral2}. Hence, for an accuracy of $\varepsilon$, the corresponding step sizes are $h_{vr1}=\mathcal{O}(\frac{\varepsilon^2 M^{2\alpha}B}{Nd})$  and $h_{vr2}=\mathcal{O}(\frac{\varepsilon M^{\alpha}\sqrt{B}}{\tau N\sqrt{d}})$, respectively. Due to the fact that the mixing time $T$ for these methods is roughly proportional to the reciprocal of step size \cite{Nilardri:2018}, it is seen that when $\varepsilon$ is small enough, one can have $h_{vr1}\ll h_{ag}$, which causes SAGA-POS converges faster than SVRG-POS (Option \uppercase\expandafter{\romannumeral1}). Similar results hold for Option \uppercase\expandafter{\romannumeral2} since the factor $\frac{1}{\sqrt{B}\tau}$ in $h_{vr2}$ would make the step size even smaller. More theoretical results are given in the SM.

%

\begin{remark}\label{remark4}
	We have provided theoretical analysis to support the statement of $\RN{1})$ in Remark \ref{remark2} . Moreover, we should also notice in SAGA-POS, stepsize $h_{ag}=\mathcal{O}(\frac{\varepsilon M^{\alpha}B}{N\sqrt{d}})$ has an extra factor, $M^{\alpha}$, compared with the step size $\mathcal{O}(\frac{\varepsilon B}{N\sqrt{d}})$ used in SAGA-LD \cite{Nilardri:2018}\footnote{For fair comparisons with our algorithms, we consider variance-reduced versions of SGLD with $M$ independent chains.}. This means SAGA-POS with more particles ($M$ is large) would outperform SAGA-LD. SVRG-POS and SVRG-POS$^+$ have similar conclusions. This theoretically supports the statements of Remark \ref{remark1} and  $\RN{1})$ in Remark \ref{remark2}. 
	Furthermore, an interesting result from the above discussion is that when $h_{vr1}=\mathcal{O}(\frac{\varepsilon^2 M^{2\alpha}B}{Nd})$ in SVRG-POS, there is an extra factor $M$ compared to the stepsize $\mathcal{O}(\frac{\varepsilon^2 B}{Nd})$ in SVRG-LD \cite{Nilardri:2018}. Since the order of $M^{2\alpha}$ is higher than $M^{\alpha}$, one expects that the improvement of SVRG-POS over SVRG-LD is much more significant than that of SAGA-POS over SAGA-LD. This conclusion will be verified in our experiments.
\end{remark}

\section{Experiments}\label{sec:exp}
We conduct experiments to verify our theory, and compare SAGA-POS, SVRG-POS and SVRG-POS$^+$  with existing representative Bayesian sampling methods with/without variance-reduction techniques, {\it e.g.} SGLD and SPOS without variance reduction; SAGA-LD, SVRG-LD and SVRG-LD$^{+}$ with variance reduction. For SVRG-POS, we focus on Option I in Algorithm~\ref{algo:algo2} to verify our theory. 

\subsection{Synthetic log-normal distribution}
We first evaluate our proposed algorithms on a log-normal synthetic data, defined as ${\small p(\mathbf{x}|\boldsymbol{\mu})=\frac{1}{\mathbf{x} \sqrt{2\pi}} \exp(-\frac{(\ln \mathbf{x} - \boldsymbol{\mu})^2}{2})}$ where $\mathbf{x}, \boldsymbol{\mu} \in \mathbb{R}^{10}$. We calculate log-MSE of the sampled ``mean'' w.r.t.\! the true value, and plot the log-MSE versus number of passes through data \cite{Nilardri:2018}, like other variance-reduction algorithms in Figure~\ref{fig:syn}, which shows that SAGA-POS and SVRG-POS converge the fastest among other algorithms. It is also interesting to see SPOS even outperforms both SAGA-LD and SVRG-LD.

\begin{figure}[t!]
	\vspace{-0.1cm}
	\centering
	\includegraphics[width=0.95\linewidth]{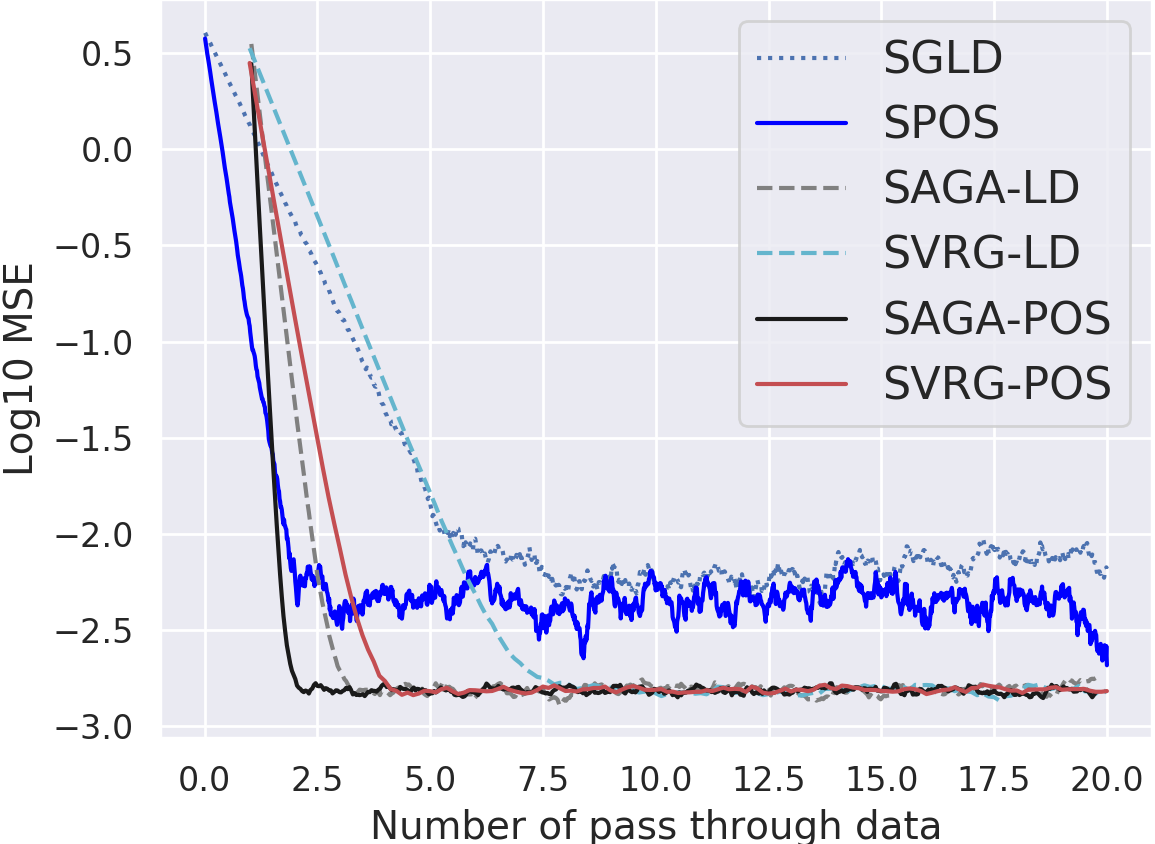}
	\vspace{-0.3cm}
	\caption{Log-MSE of the mean parameter versus the number of dataset pass.}
	\label{fig:syn}
	\vspace{-0.6cm}
\end{figure}

\subsection{Bayesian logistic regression}
Following related work such as \cite{DubeyRPSX:nips16}, we test the proposed algorithms for Bayesian-logistic-regression (BLR) on four publicly available datasets from the UCI machine learning repository: $Australian$ (690-14), $Pima$ (768-8), $Diabetic$ (1151-20) and $Susy$ (100000-18), where $(N-d)$ means a dataset of $N$ data points with dimensionality $d$. The first three datasets are relatively small, and the last one is a large dataset suitable to evaluate scalable Bayesian sampling algorithms. 


Specifically, consider a dataset $\{X_i, y_i\}_{i=1}^N$ with $N$ samples, where $X_i \in \mathbb{R}^d$ and $y_i \in \{0, 1\}$. The likelihood of a BLR model is written as $p(y_i=1|X_i, \alpha)=\sig(\alpha^TX_i)$ with regression coefficient $\alpha \in \mathbb{R}^d$, which is assumed to be sampled from a standard multivariate Gaussian prior $\mathcal{N}(0, I)$ for simplicity. 
The datasets are split into 80\% training data and 20\% testing data. Optimized constant stepsizes are applied for each algorithm via grid search. 
Following existing work, we report testing accuracy and log-likelihood versus the number of data passes for each dataset, averaging over 10 runs with 50 particles. The minibatch size is set to 15 for all experiments.

\subsubsection{Variance-reduced SPOS versus SPOS}
We first compare SAGA-POS, SVRG-POS and SVRG-POS$^+$ with SPOS without variance reduction proposed in \cite{ZhangZC:18}. The testing accuracies and log-likelihoods versus number of passes through data on the four datasets are plotted in Figure~\ref{fig:blr2}. It is observed that SAGA-POS converges faster than both SVRG-POS and SVRG-POS$^+$, all of which outperform SPOS significantly. On the largest dataset {\em SUSY}, SAGA-POS starts only after one pass through data, which then converges quickly, outperforming other algorithms. And SVRG-POS$^+$ outperforms SVRG-POS due to the dataset {\em SUSY} is so large. All of these phenomena are consistent with our theory.

\begin{figure}
	\centering
	\begin{minipage}{\linewidth}
		\centering
		\includegraphics[width=0.9\linewidth]{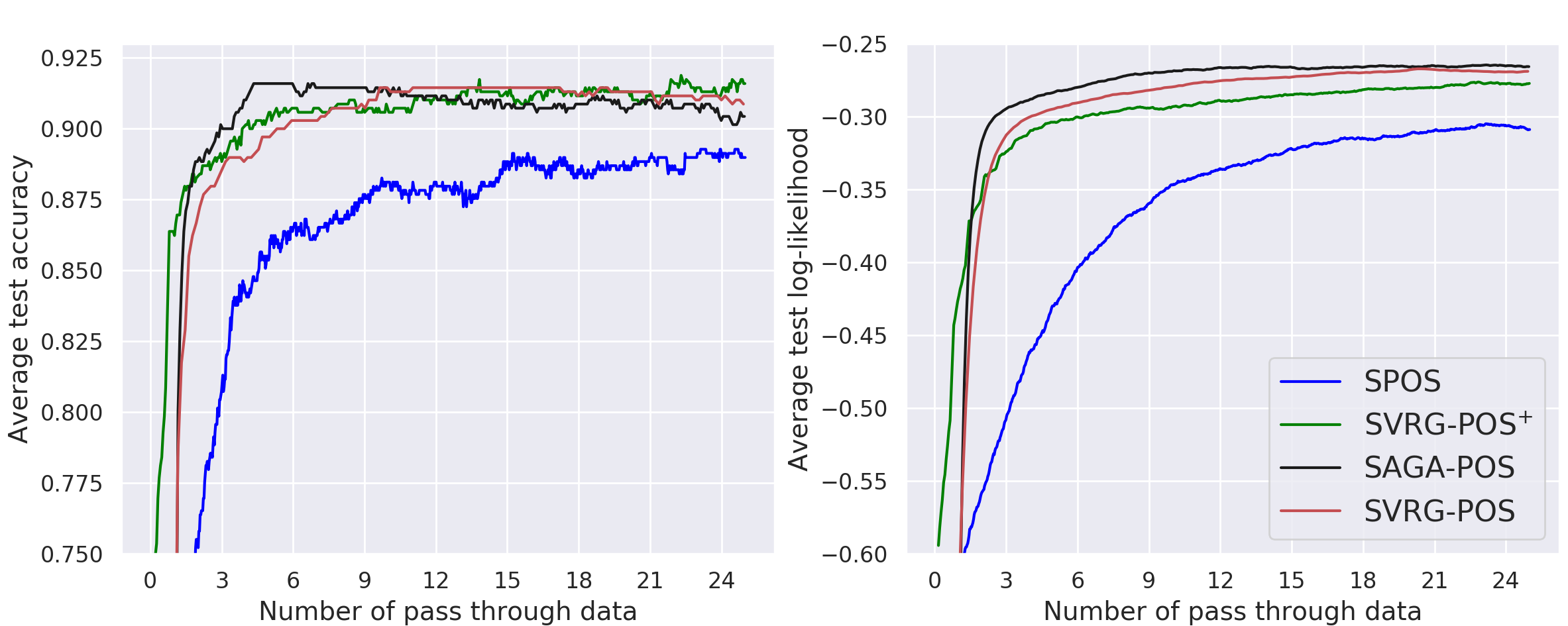}
	\end{minipage}
	\begin{minipage}{\linewidth}
		\centering
		\includegraphics[width=0.9\linewidth]{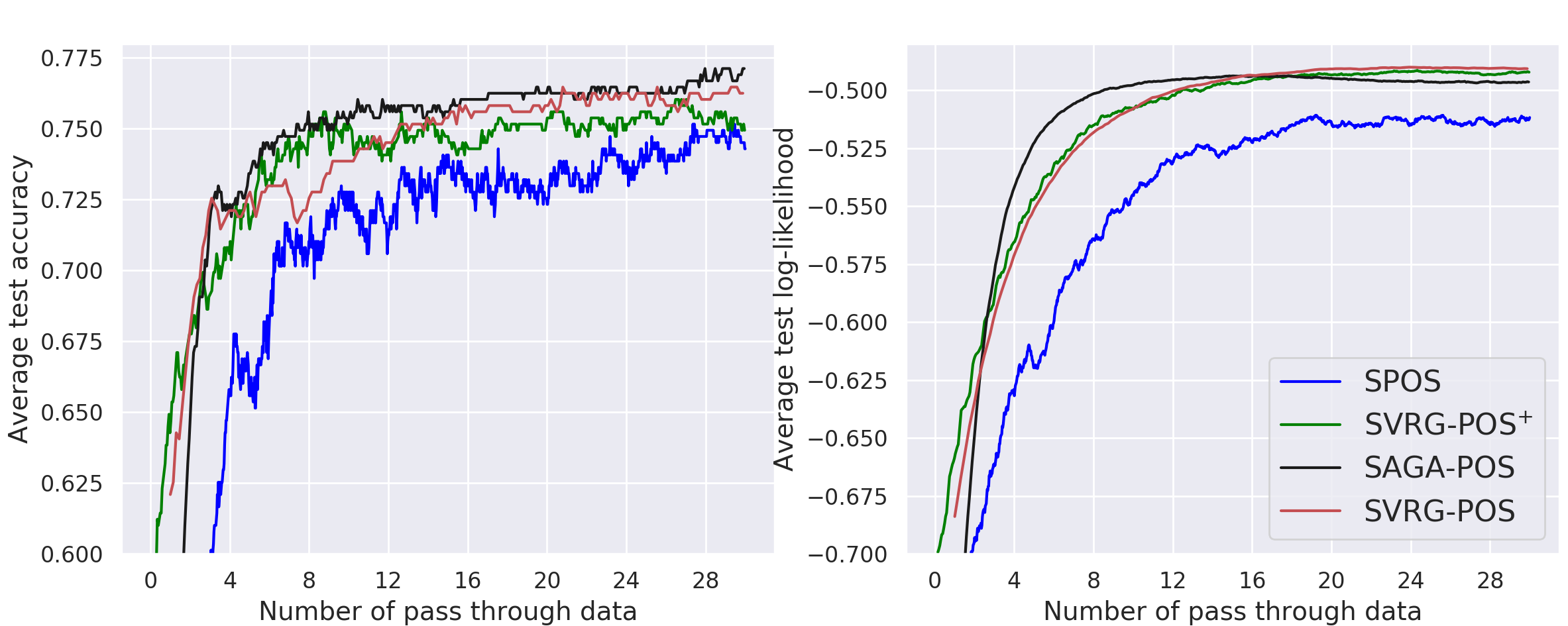}
	\end{minipage}
	\begin{minipage}{\linewidth}
		\centering
		\includegraphics[width=0.9\linewidth]{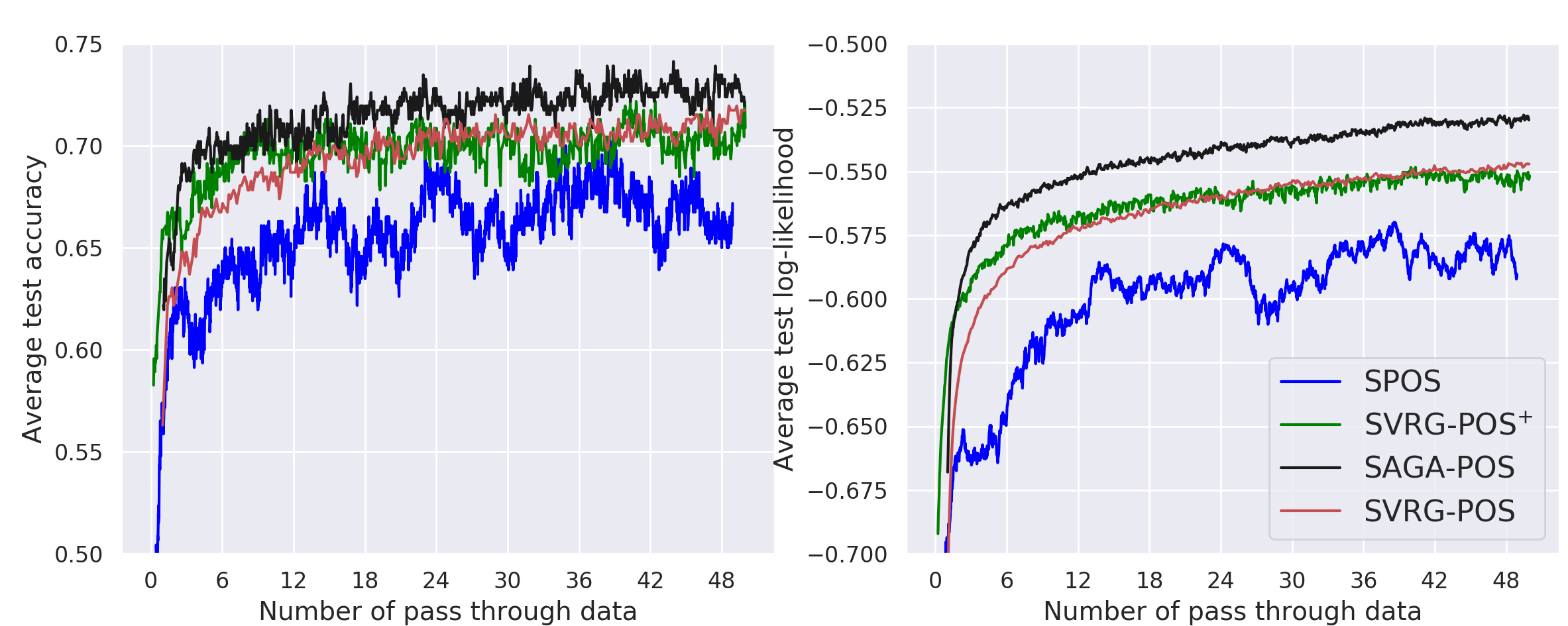}
	\end{minipage}
	\begin{minipage}{\linewidth}
		\centering
		\includegraphics[width=0.9\linewidth]{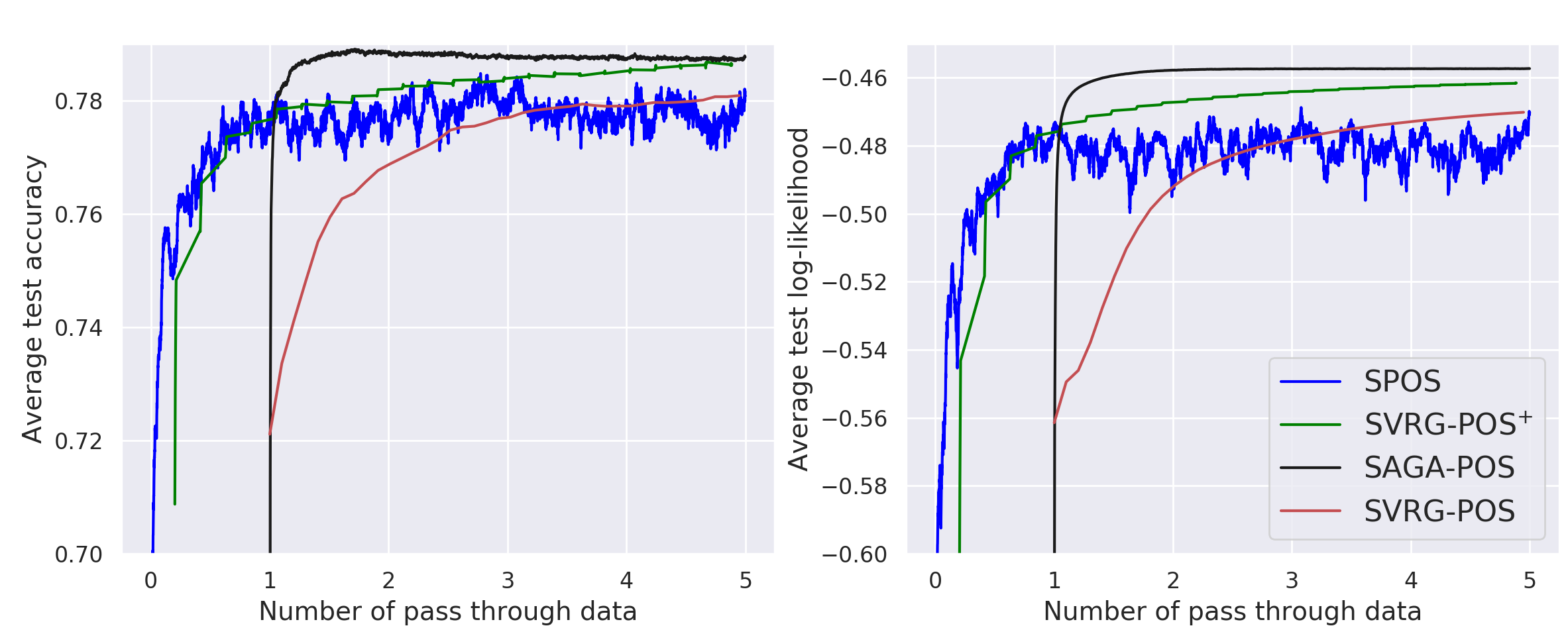}
	\end{minipage}
	\caption{Testing accuracy and log-likelihood vs the number of data pass for SPOS and its variance-reduction variants. From top to bottom: AUSTRALIAN, PIMA, DIABETIC, SUSY datasets.}
	\label{fig:blr2}
\end{figure}

\subsubsection{Variance-reduced SPOS versus variance-reduced SGLD}
Next we compare the three variance-reduced SPOS algorithms with its SGLD counterparts, {\it i.e.}, SAGA-LD, SVRG-LD and SVRG-LD$^+$. The results are plotted in Figure~\ref{fig:blr1}. Similar phenomena are observed, where both SAGA-POS and SVRG-POS outperform SAGA-LD and SVRG-LD, respectively, consistent with our theoretical results discussed in Remark~\ref{remark1} and \ref{remark2}. Interestingly, in the PIMA dataset case, SVRG-LD is observed to perform even worse (converge slower) than standard SGLD. Furthermore, as discussed in Remark~\ref{remark4}, our theory indicates that the improvement of SVRG-POS over SVRG-LD is more significant than that of SAGA-POS over SAGA-LD. This is indeed true by inspecting the plots in Figure~\ref{fig:blr1}.

\begin{figure}
	\centering
	\begin{minipage}{\linewidth}
		\centering
		\includegraphics[width=0.9\linewidth]{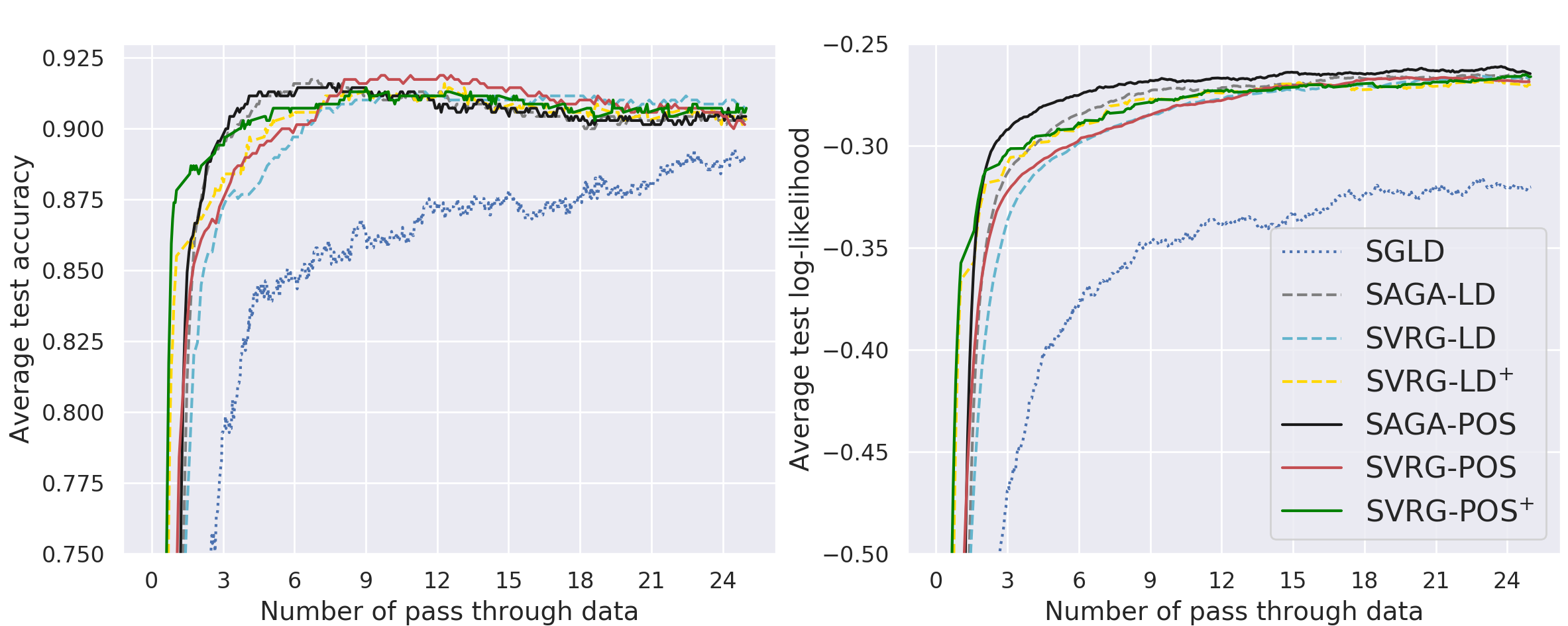}
	\end{minipage}
	\begin{minipage}{\linewidth}
		\centering
		\includegraphics[width=0.9\linewidth]{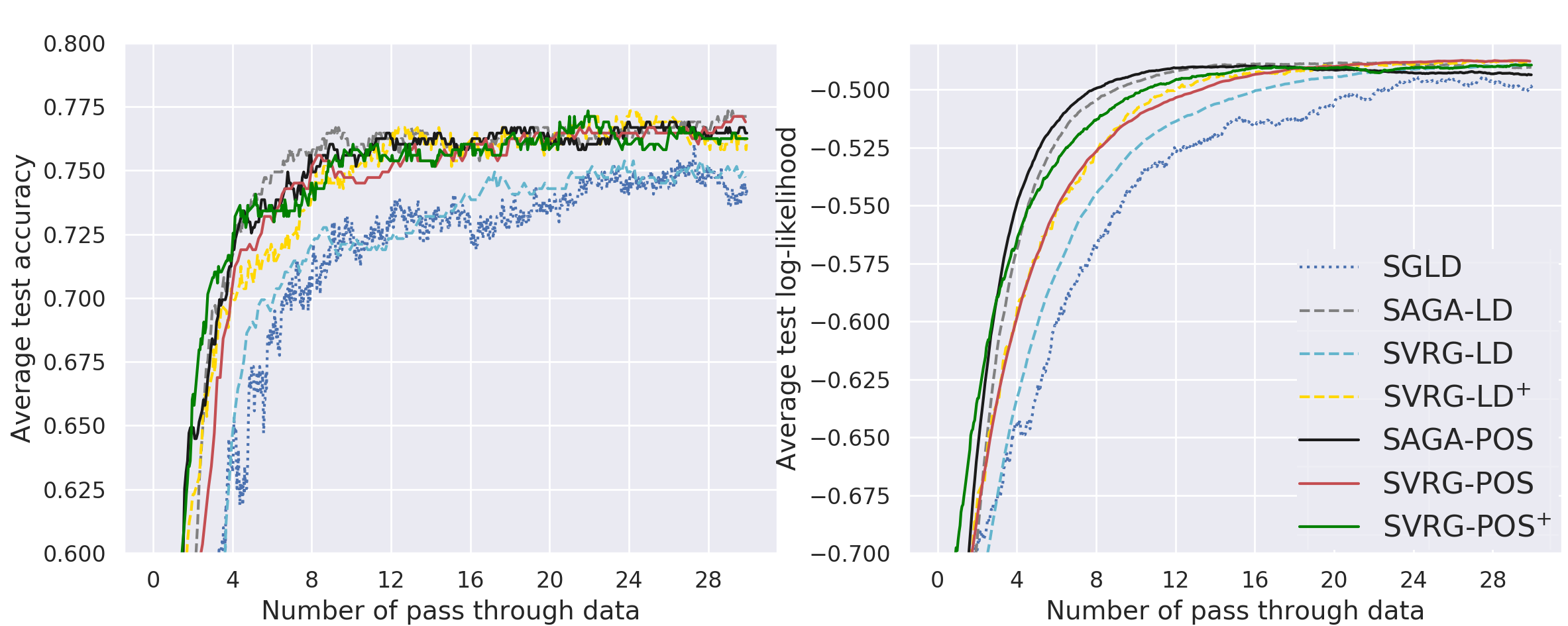}
	\end{minipage}
	\begin{minipage}{\linewidth}
		\centering
		\includegraphics[width=0.9\linewidth]{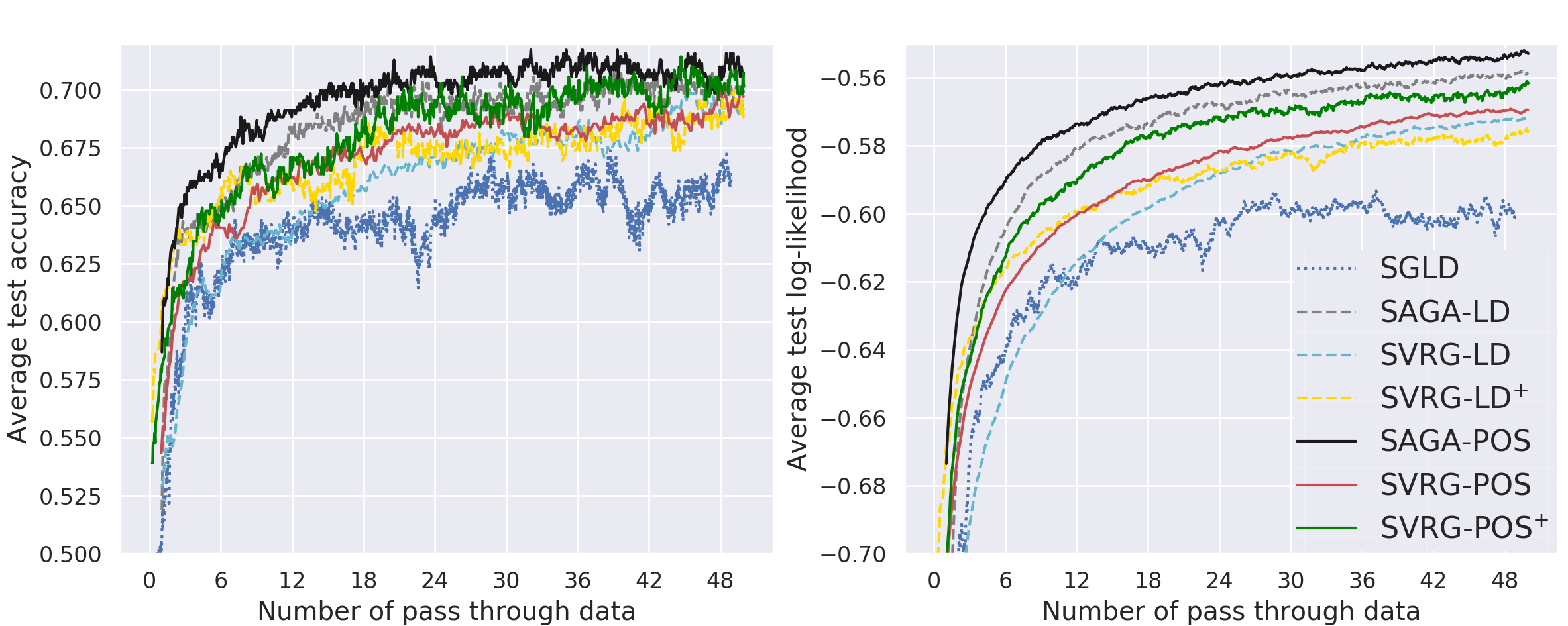}
	\end{minipage}
	\begin{minipage}{\linewidth}
		\centering
		\includegraphics[width=0.9\linewidth]{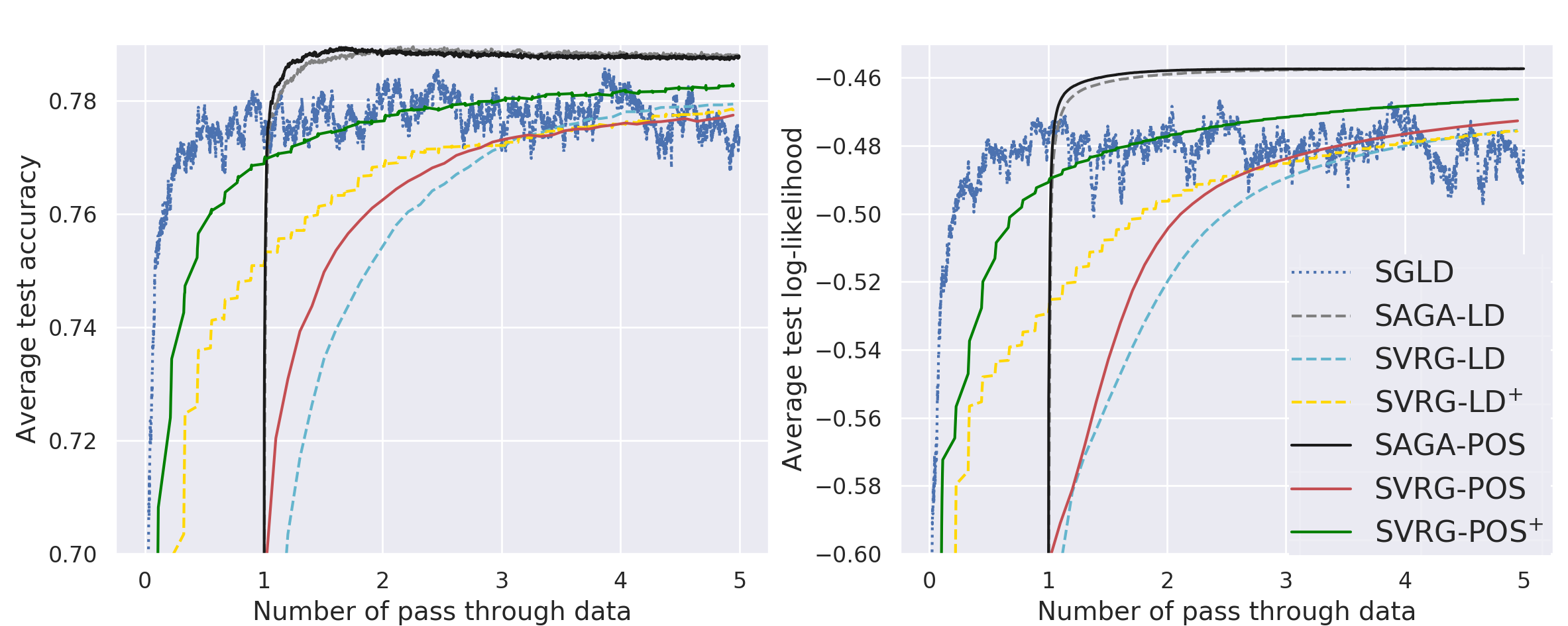}
	\end{minipage}
	\caption{Testing accuracy and log-likelihood versus the number of dataset pass for variance-reduced SPOS and SGLD. From top to bottom: AUSTRALIAN, PIMA, DIABETIC, SUSY datasets.}
	\label{fig:blr1}
	\vspace{-0.5cm}
\end{figure}

\subsubsection{Impact of number of particles}

Finally we examine the impact of number of particles to the convergence rates. As indicated by Theorems~\ref{thm:thm1}-\ref{thm:thm3}, for a fixed number of iterations $T$, the convergence error in terms of 2-Wasserstein distance decreases with increasing number of particles. To verify this, we run SAGA-POS and SVRG-POS for BLR with the number of particles ranging between $\{1, 2, 4, 8, 16\}$. The test log-likelihoods versus iteration numbers are plotted in Figure~\ref{fig:blr3}, demonstrating consistency with our theory.

\begin{figure}[t!]
	\centering
	\begin{minipage}{\linewidth}
		\centering
		\includegraphics[width=0.9\linewidth]{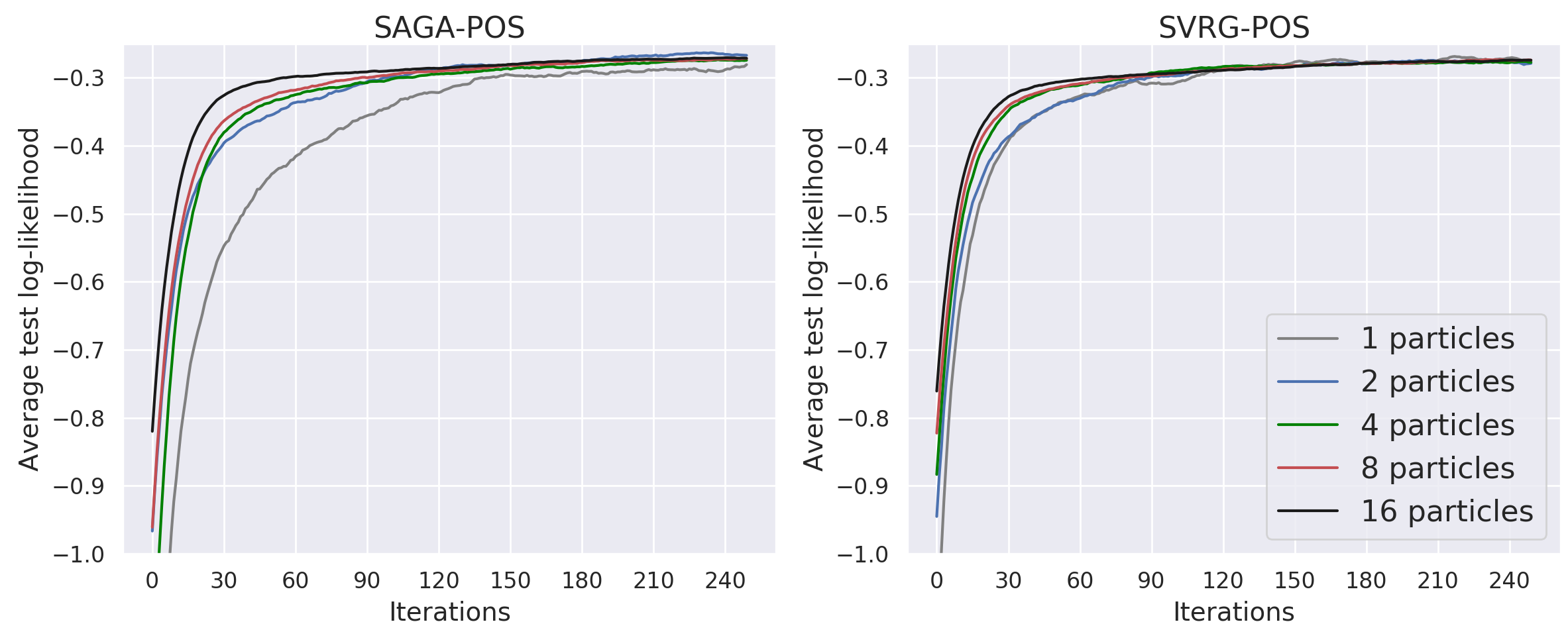}
	\end{minipage}
	\begin{minipage}{\linewidth}
		\centering
		\includegraphics[width=0.9\linewidth]{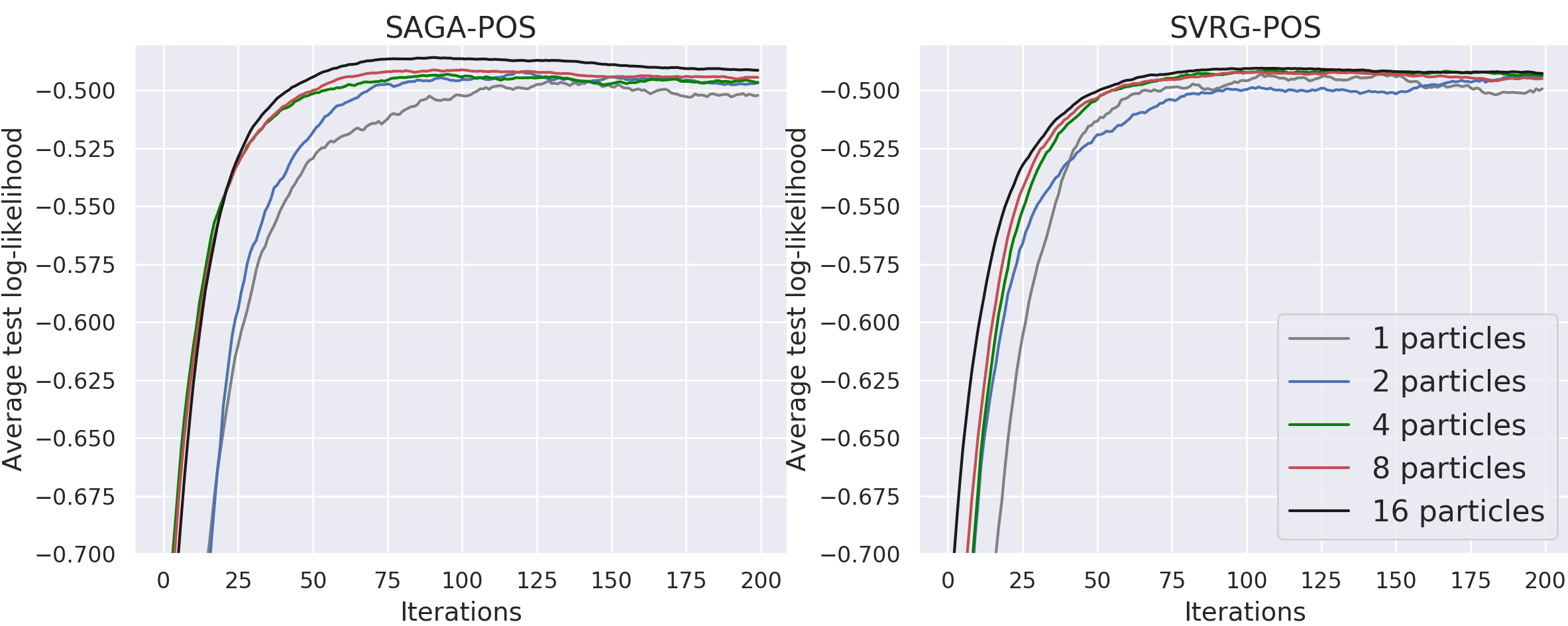}
	\end{minipage}
	\caption{Testing log-likelihood versus number of iterations with different number of particles for variance-reduced SPOS. Top: Australian; Bottom: Pima datasets.}
	\label{fig:blr3}
	\vspace{-0.5cm}
\end{figure}

\section{Conclusion}

We propose several variance-reduction techniques for stochastic particle-optimization sampling, and for the first time, develop nonasymptotic convergence theory for the algorithms in terms of 2-Wasserstein metrics. Our theoretical results indicate the improvement of convergence rates for the proposed variance-reduced SPOS compared to both standard SPOS and the variance-reduced SGLD algorithms. Our theory is verified by a number of experiments on both synthetic data and real data for Bayesian Logistic regression. Leveraging both our theory and empirical findings, we recommend the following algorithm choices in practice: $\RN{1})$ SAGA-POS is preferable when storage is not a concern; $\RN{2})$ SVRG-POS is a better choice when storage is a concern and full gradients are feasible to calculate; $\RN{3})$ Otherwise, SVRG-POS$^+$ is a good choice and works well in practice.

\newpage
\bibliographystyle{alpha}
\bibliography{reference}

\appendix
\section{More details about the notations}
\begin{itemize}
	\item If you read this paper carefully, you may notice the different use of $\thetab$ and $\theta$. $\thetab$ is mostly used for the interpretation of the theory. However, $\theta$ is only used for the interpretation of algorithms, which means  $\theta$ often appears with $k$ (which stands for the $k$th interation ) like $\theta_k$. We design these differences to help you have a better understanding of our results. \\
	The above rules still apply for the results in Appendix.  
	\item The symbol $\boldsymbol{1}(H_1 \leq H_2)$ in  Theorem \ref{thm:thm3} means 
	\begin{equation}
	\boldsymbol{1}(H_1 \leq H_2)=
	\begin{cases}
	1& H_1 \leq H_2\\
	0& H_1 > H_2
	\end{cases}
	\end{equation}\\ and the symbol $H_3 \wedge H_4$ means $\min\{H_3,H_4\}$\\\
	\item The relationship between RBF kernel $\kappa(\thetab,\thetab^{\prime})=\exp(-\frac{\| \thetab-\thetab^{\prime} \|^2}{2\eta^2})$  and the function $K(\thetab)=\exp(-\frac{\| \thetab \|^2}{2\eta^2})$ can be interpreted as $\kappa(\thetab,\thetab^{\prime})=K(\thetab-\thetab^{\prime})$ in detail.\\
	
	We moved the above details about the notations to the appendix due to the space limit.
\end{itemize}

\section{Convergence guarantees for SAGA-LD, SVRG-LD and SVRG-LD$^+$}
In this section we present the Convergence guarantees for SAGA-LD, SVRG-LD and SVRG-LD+ from \cite{Nilardri:2018,DIFAN:2018}
\begin{assumption}\label{ass:assnew}
	\begin{itemize}
		\item(Sum-decomposable) The $ F(\thetab)$ is decomposable i.e. $F(\thetab)=\sum_{j=1}^NF_j(\thetab)$
		\item(Smoothness)  $F(\thetab)$ is Lipschitz continuous with some positive constant, i.e. for all $\thetab_1, \thetab_2 \in \mathbb{R}^{d}$, $\|F(\thetab_1)-F(\thetab_2)\|\leq L_F \left\|\thetab_1- \thetab_2\right\|$ 
		\item(Strong convexity)  $F(\thetab)$ is a $m_F$-strongly convex function, i.e.
		$\left (F(\thetab_1)-F(\thetab_2 ) \right)(\thetab_1-\thetab_2)	\geq
		m_F\left\| \thetab_1-\thetab_2 \right\|$
		\item(Hessian Lischitz) There exits such a positive constant such that $\left\|\nabla F(\thetab_1)-\nabla F(\thetab_2)\right\| \leq D_F\left\|\thetab_1-\thetab_2\right\|$
	\end{itemize}
	
\end{assumption}
\begin{assumption}\label{ass:assnew1}
	(Bound Variance)\footnote{This assumption is a little different from that in \cite{DIFAN:2018} since we adopt different definition of $F_j$} There exits a constant $\sigma \geq 0 $, such that for all j
	\begin{align*}
	\mathbb{E}[\|F_j(\thetab)-\frac{1}{N}\sum \limits_{j=1}^{N}F_j(\thetab)\|^2]\leq d\sigma^2/N^2
	\end{align*}
\end{assumption}
\begin{theorem}\label{thm:citeThe1}
	Under Assumption \ref{ass:assnew}, let the step size $h < \frac{B}{8NL_F}$ and the batch size $B \geq 9$, then we can have the bound for $\mathcal{W}_2(\mu_T,\mu^*)$ in the SAGA-LD algorithm
	
	\begin{align}
	\mathcal{W}_2(\mu_T,\mu^*)\leq &
	5 \exp(-\frac{m_F h}{4}T)\mathcal{W}_2(\mu_0,\mu^*)+ \nonumber\\ 
	&\frac{2hD_Fd}{m_F}+\frac{2h{L_F}^\frac{3}{2}\sqrt{d}}{m_F}+\frac{24hL_F\sqrt{dN}}{\sqrt{m_F}B} \nonumber
	\end{align}
\end{theorem}

\begin{theorem} \label{thm:citeThe2}
	Under Assumption \ref{ass:assnew}, if we choose Option $\uppercase\expandafter{\romannumeral1}$ and set the step size $h < \frac{1}{8L_F}$, the batch size $B \geq 2$ and the epoch length $\tau \geq \frac{8}{m_F h}$, then we can have the bound for all T mod $\tau$ =0 in the SVRG-LD algorithm
	\begin{align}
	\mathcal{W}_2(\mu_T,\mu^*)\leq & \exp(-\frac{m_Fh}{56}T)\frac{\sqrt{L_F}}{\sqrt{m_F}} \mathcal{W}_2(\mu_0,\mu^*)+\nonumber \\
	&\frac{2hD_Fd}{m_F}+\frac{2h{L_F}^\frac{3}{2}\sqrt{d}}{m_F}+\frac{64L_F^{\frac{3}{2}}\sqrt{hd}}{m_F\sqrt{B}}\nonumber
	\end{align}
	If we choose Option $\uppercase\expandafter{\romannumeral2}$ and set the step size $h < \frac{\sqrt{B}}{4\tau C_2}$, then we can have the bound for all T in the SVRG-LD algorithm
	\begin{align}
	\mathcal{W}_2(\mu_T,\mu^*)\leq &\exp(-\frac{m_Fh}{4}T)\mathcal{W}_2(\mu_0,\mu^*)+\nonumber\\
	&\frac{\sqrt{2}h D_Fd}{m_F}+\frac{5h{L_F}^\frac{3}{2}\sqrt{d}}{m_F}+\frac{9hL_F\tau\sqrt{d}}{\sqrt{Bm_F}}\nonumber
	\end{align}
\end{theorem}
\begin{theorem} \label{thm:citeThe3}
	Under Assumption \ref{ass:assnew} and Assumption \ref{ass:assnew1}, if we set the step size $h \leq min\{(\frac{BC_3 }{24{C_2}^4\tau^2})^{\frac{1}{3}},\frac{1}{6\tau({C_5}^2/b+C_2)}\}$, then we can have the bound for all T in the SVRG-LD$^+$ algorithm. 
	\begin{align}
	&\mathcal{W}_2(\mu_T,\mu^*)\leq (1-hm_F/4)^T\mathcal{W}_2(\mu_0,\mu^*)+\nonumber\\
	&\frac{3\sigma d^{1/2}}{m_Fb^{1/2}}\boldsymbol{1}(b\leq N)+\frac{2hD_4 d}{m_F}+\frac{2h{L_F}^{3/2}d^{1/2}}{m_F}\nonumber\\
	&+\frac{4hL_F(\tau d)^{1/2}\wedge 3h^{1/2}d^{1/2}\sigma}{\sqrt{Bm_F}}\nonumber
	\end{align}
\end{theorem}

\section{Proof of the theorems in Section 4}
In this section, we give proofs to the theorems in Section 4. We are sorry that the proof of our theorems is a little long since we want to make it more easy to understand. However, this does not affect that fact that our proof is credible. Our proof is based on the idea of \cite{ZhangZC:18} and borrow some results from \cite{Nilardri:2018,DIFAN:2018}\\
\begin{align} \label{eq:particleapp}
\mathrm{d}\thetab_{t}^{(i)} =&-\beta^{-1}F(\thetab_t^{(i)})\mathrm{d}t - \frac{1}{M}\sum_{q=1}^{M}K(\thetab_{t}^{(i)} - \thetab_{t}^{(q)})F(\thetab_{t}^{(q)})\mathrm{d}t \nonumber \\
&+\frac{1}{M}\sum_{q=1}^{M}\nabla K(\thetab_{t}^{(i)} - \thetab_{t}^{(q)})\mathrm{d}t+ \sqrt{2\beta^{-1}}\mathrm{d}\mathcal{W}_t^{(i)}~~~~\forall i~
\end{align}
As mention is Section \ref{SPOS} we denote the distribution of $\thetab_{t}^{(i)}$ in Eq.(\ref{eq:particleapp}) as $\nu_t$. From the proof of Theorem 3 and Remark 1 in  \cite{ZhangZC:18} we can derive that  
\begin{align}
\mathcal{W}_2(\nu_{\infty},\mu^*)\leq \frac{H_{\nabla K}+H_F}{\sqrt{2M}(\beta^{-1}-3H_FL_K-2L_F)}
\end{align}
In order to bound $\mathcal{W}_2(\mu_{T},\mu^*)$, we need to bound $\mathcal{W}_2(\mu_{T},\nu_{\infty})$ next. Now we borrow the idea in  \cite{ZhangZC:18}  
, concatenating the particles at each time into a single vector representation, 
We define a new parameter at time $t$ as $\Thetab_t\triangleq [\thetab_{t}^{(1)}, \cdots, \thetab_{t}^{(M)}] \in \mathbb{R}^{Md}$. Consequently, $\Thetab_t$ is driven by the following linear SDE:
\begin{align}\label{eq:extendP}
\mathrm{d}\Thetab_t = -F^{\Thetab}(\Thetab_t)\mathrm{d}t + \sqrt{2\beta^{-1}}\mathrm{d}\mathcal{W}_t^{(Md)}~,
\end{align}
$F^{\Thetab}(\Thetab_t) \triangleq [\beta^{-1}F(\thetab_{t}^{(1)})-\frac{1}{M}\sum_{q=1}^M\nabla K(\thetab_t^{(1)}-\thetab_t^{(q)})\\+\frac{1}{M}\sum_{q=1}^{M}K(\thetab_{t}^{(1)} - \thetab_{t}^{(q)})F(\thetab_{t}^{(q)}), \cdots, \beta^{-1}F(\thetab_{t}^{(M)})- \frac{1}{M}\sum_{q=1}^M\nabla K(\thetab_t^{(M)}-\thetab_t^{(q)})+\frac{1}{M}\sum_{q=1}^{M}K(\thetab_{t}^{(M)} - \thetab_{t}^{(q)})F(\thetab_{t}^{(q)})]$ is a vector function $\mathbb{R}^{Md}\rightarrow \mathbb{R}^{Md}$, and $\mathcal{W}_t^{(Md)}$ is Brownian motion of dimension $Md$. \\

Now we define the $F^{\Thetab}_j(\Thetab_t) \triangleq [\beta^{-1}F_j(\thetab_{t}^{(1)})-\frac{1}{MN}\sum_{q=1}^M\nabla K(\thetab_t^{(1)}-\thetab_t^{(q)})+\frac{1}{M}\sum_{q=1}^{M}K(\thetab_{t}^{(1)} - \thetab_{t}^{(q)})F_j(\thetab_{t}^{(q)}), \cdots, \beta^{-1}F_j(\thetab_{t}^{(M)})-\frac{1}{MN}\sum_{q=1}^M\nabla K(\thetab_t^{(M)}\\-\thetab_t^{(q)})+\frac{1}{M}\sum_{q=1}^{M}K(\thetab_{t}^{(M)} - \thetab_{t}^{(q)})F_j(\thetab_{t}^{(q)})]$. We can find the $F^{\Thetab}{(\Thetab_t)}$ and $F^{\Thetab}_j(\Thetab_t)$ defined above satisfy the following theorem.\\

\begin{theorem}\label{the:them4}
	\begin{itemize}
		\item(Sum-decomposable) The $ F^{\Thetab}(\Thetab)$ is decomposable i.e. $F^{\Thetab}(\Thetab)=\sum_{j=1}^NF^{\Thetab}_j(\Thetab)$
		\item(Smoothness)  $F^{\Thetab}$ is Lipschitz continuous with some positive constant, i.e. for all $\Thetab_1, \Thetab_2\in \mathbb{R}^{Md}$,$\|F^{\Thetab}(\Thetab_1)-F^{\Thetab}(\Thetab_2)\|\leq \sqrt{2(\beta^{-1}L_F+2L_K H_F+H_K L_F+L_{\nabla K})^2+2}\left\|\Thetab_1-\Thetab_2\right\|$ 
		\item(Strong convexity)  $F^{\Thetab}$ is a $(\beta^{-1}m_F--2L_F-3H_FL_K)$-strongly convex function, i.e.
		$\left(F^{\Thetab}(\Thetab_1)-F^{\Thetab}(\Thetab_2)\right)(\Thetab_1-\Thetab_2)	\leq
		(\beta^{-1}m_F-2L_F-3H_FL_K)\left\|\Thetab_1-\Thetab_2\right\|$
		\item(Hessian Lischitz) 	The function $F^{\Thetab}$ is Hessian Lipschitz, i.e., $\left\|\nabla F^{\Thetab}(\Thetab_1)-\nabla F^{\Thetab}(\Thetab_2)\right\| \leq (\beta^{-1}D_F+4D_{\nabla^2K}+4H_FL_{\nabla K}+2L_FH_{\nabla K}+
		2H_FL_{K}+L_FH_{K})\left\|\Thetab_1-\Thetab_2\right\|$
		\item(Bound Variance) There exits a constant, $\sigma \geq 0$, such that for all $j$,
		\begin{align*}
		\mathbb{E}[\|F^{\Thetab}_j(\Thetab)-\frac{1}{N}\sum \limits_{j=1}^{N}F^{\Thetab}_j(\Thetab)\|^2]\leq Md(2\beta^{-1}+2H_K^2)\sigma^2/ N^2
		\end{align*}
	\end{itemize}
\end{theorem}

\begin{proof}
	\begin{itemize}
		\item The sum-decomposable property of $F^{\Thetab}(\Thetab)$ is easy to verify. And the smoothness property of $F^{\Thetab}$ can be derived directly from the proof of the Lemma 13 in \cite{ZhangZC:18}.
		\item (Strong convexity)
		\begin{align}(F^{\Thetab}(\Thetab_1)-F^{\Thetab}(\Thetab_2))(\Thetab_1-\Thetab_2)=
		\\
		\frac{1}{M}\sum_{i,q}^{M} (\xi_{iq}^{1} + \xi_{iq}^{2}  + \xi_{iq}^{3}  +\xi_{iq}^{4})\nonumber
		\end{align}
		where \begin{align}
		\xi_{iq}^{1}  = \beta^{-1}\left(F(\thetab_{1}^{(i)}) - F({\thetab}_2^{(i)})\right)\cdot \left(\thetab_{1}^{(i)} - {\thetab}_2^{(i)}\right) \nonumber
		\end{align}
		\begin{align}
		\xi_{iq}^{2} = -\left(\nabla K(\thetab_1^{(i)} - \thetab_1^{(q)}) - \nabla K(\thetab_2^{(i)} - \thetab_2^{(q)})\right)\cdot \left(\thetab_1^{(i)} - \thetab_2^{(i)}\right)\nonumber
		\end{align}
		\begin{align}
		\xi_{iq}^{3}  = \left(F(\thetab_1^{(q)}) K(\thetab_1^{(i)}-\thetab_1^{(q)}) - F(\thetab_2^{(q)}) K(\thetab_1^{(i)}- \thetab_1^{(q)}) \right)\cdot \nonumber\\
		\left(\thetab_1^{(i)} - {\thetab}_2^{(i)}\right)\nonumber
		\end{align}
		\begin{align}
		\xi_{iq}^{4}  = \left(F(\thetab_2^{(q)}) K(\thetab_1^{(i)}-\thetab_1^{(q)}) - F(\thetab_2^{(q)}) 
		K(\thetab_2^{(i)}- \thetab_2^{(q)}) \right)\cdot \nonumber\\
		\left(\thetab_1^{(i)} - {\thetab}_2^{(i)}\right)\nonumber
		\end{align}
		For the $\xi_{iq}^{1}$ terms, applying the convex condition for $F$, we have
		\begin{align}
		\sum_{iq}\xi_{iq}^{1}&= \sum_{iq} \beta^{-1}\left(F(\thetab_{1}^{(i)}) - F({\thetab}_2^{(i)}\right)\cdot \left(\thetab_{1}^{(i)} - {\thetab}_2^{(i)}\right) \nonumber \\
		&\geq \beta^{-1}m_FM\sum_i\left\|\thetab_{1}^{(i)} - {\thetab}_2^{(i)}\right\|^2
		\end{align} 
		
		For the $\xi_{iq}^{2}$ term, applying the concave condition for $W$ and $\nabla W$ is odd, we have $\sum_{iq}\xi_{iq}^{2}=-$ 
		\begin{align}
		&\sum_{iq}^{M}\left(\nabla K(\thetab_1^{(i)} - \thetab_1^{(q)}) - \nabla K(\thetab_2^{(i)} - \thetab_2^{(q)})\right)\cdot \left(\thetab_1^{(i)} - \thetab_2^{(i)}\right) \nonumber
		\end{align}
		\begin{align}
		=&-\frac{1}{2}\sum_{iq}^{M}\sum_{iq}^{M}\left(\nabla K(\thetab_1^{(i)} - \thetab_1^{(q)}) - \nabla K(\thetab_2^{(i)} - \thetab_2^{(q)})\right) \nonumber\\
		&\cdot \left(\thetab_1^{(i)} - \thetab_1^{(q)}-(\thetab_2^{(i)} - \thetab_2^{(q)})\right)\nonumber
		\end{align}
		\begin{align}
		&\geq \frac{1}{2} m_K\sum_{iq}^{M}\left\|\thetab_1^{(i)} - \thetab_1^{(q)}-(\thetab_2^{(i)} - \thetab_2^{(q)})\right\|^2 \geq 0
		\end{align}
		For the $\xi_{iq}^{3}$ terms, after applying the $L_F$-Lipschitz property of $F$, we have $\sum_{iq}\xi_{iq}^{3} =$
		\begin{align}
		\sum_{iq}(F(\thetab_1^{(q)}) K(\thetab_1^{(i)}
		-\thetab_1^{(q)})-F(\thetab_2^{(q)})K(\thetab_1^{(i)}- \thetab_1^{(q)})) \nonumber\\
		\cdot\left(\thetab_1^{(i)} - \thetab_2^{(i)}\right) 
		\geq -\sum_{iq}L_F \left\|\thetab_{1}^{(q)} - \thetab_2^{(q)}\right\| \left\|\thetab_{1}^{(i)} - \thetab_2^{(i)}\right\|  \nonumber\\
		\geq -2L_F M\sum_{i} \left \|\thetab_{1}^{(i)} - \thetab_2^{(i)}\right\|^2~\
		\end{align}
		For the $\xi_{iq}^{4}$ terms, we have $\sum_{iq} \xi_{iq}^{4}=$
		
		{\begin{align}
			\sum_{iq}(F(\thetab_2^{(q)}) K(\thetab_1^{(i)}-\thetab_1^{(q)})
			- F(\thetab_2^{(q)}) K(\thetab_2^{(i)}- \thetab_2^{(q)})) \nonumber \\ 
			\cdot \left(\thetab_1^{(i)} - {\thetab}_2^{(i)}\right)  \nonumber \\
			\geq -H_FL_K\sum_{iq}\left\|\thetab_1^{(i)} - \thetab_1^{(q)}-(\thetab_2^{(i)} - \thetab_2^{(q)})\right\|\left\|\thetab_1^{(i)} - \thetab_2^{(i)}\right\|  \nonumber \\
			\geq -3H_FL_KM \sum_{i}\left\|\thetab_1^{(i)} - \thetab_2^{(i)}\right\| ^2
			\end{align}}
		Then we finally arrive at: 
		\begin{align}
		&\left(F^{\Thetab}(\Thetab_1)-F^{\Thetab}(\Thetab_2)\right)(\Thetab_1-\Thetab_2)	 \nonumber \\
		&\geq (\beta^{-1}m_F-2L_F-3H_FL_K)\sum_{i}\left\|\thetab_1^{(i)} - \thetab_2^{(i)}\right\| \nonumber
		\\
		&\geq (\beta^{-1}m_F-2L_F-3H_FL_K)\left\|\Thetab_1-\Thetab_2\right\|
		\end{align}
		
		\item Now, we will prove the fourth result:
		\begin{align} 
		&\left\|\nabla F^{\Thetab}(\Thetab_1)-\nabla F^{\Thetab}(\Thetab_2)\right\| \nonumber  \\
		&\leq \beta^{-1}\sum_{i=1}^{M}\left\|\nabla F(\thetab_1^{(i)})-\nabla F(\thetab_2^{(i)})\right\|+ 
		\sum_{i=1}^{M}\frac{2}{M}\sum_{q=1}^{M}\left\|\nabla^2K(\thetab_1^{(i)}-\thetab_1^{(q)})-\nabla^2K(\thetab_2^{(i)}-\thetab_2^{(q)})\right\|+ \nonumber \\
		&\frac{2}{M}\sum_{i=1}^{M}\sum_{q=1}^{M}\|\nabla K(\thetab_1^{(i)}-\thetab_1^{(q)})F(\thetab_1^{(q)})- \nabla K(\thetab_2^{(i)}-\thetab_2^{(q)})F(\thetab_2^{(q)})\| \nonumber \\
		&+\sum_{i=1}^{M}\sum_{q=1}^{M}\frac{1}{M}\|K(\thetab_1^{(i)}-\thetab_1^{(q)})\nabla F(\thetab_1^{(q)})-K(\thetab_2^{(i)}-\thetab_2^{(q)})\nabla F(\thetab_2^{(q)})\|\nonumber \\
		&\leq \sum_{i=1}^{M}\beta^{-1}D_F\|\thetab_1^{(i)}-\thetab_2^{(i)}\|+4D_{\nabla^2K}\sum_{i=1}^{M}\|\thetab_1^{(i)}-\thetab_2^{(i)}\|+\nonumber\\
		&\frac{2}{M}\sum_{i=1}^{M}\sum_{q=1}^{M}\|\nabla K(\thetab_1^{(i)}-\thetab_1^{(q)})F(\thetab_1^{(q)})-\nabla K(\thetab_2^{(i)}-\thetab_2^{(q)})F(\thetab_1^{(q)})\|+ \nonumber \\ 
		&\frac{2}{M}\sum_{i=1}^{M}\sum_{q=1}^{M}\|\nabla K(\thetab_2^{(i)}-\thetab_2^{(q)})F(\thetab_1^{(q)})- \nabla K(\thetab_2^{(i)}-\thetab_2^{(q)})F(\thetab_2^{(q)})\|\nonumber \\
		&+\frac{1}{M}\sum_{i=1}^{M}\sum_{q=1}^{M}\| K(\thetab_1^{(i)}-\thetab_1^{(q)})F(\thetab_1^{(q)})- K(\thetab_2^{(i)}-\thetab_2^{(q)})F(\thetab_1^{(q)})\|+ \nonumber \\
		&\frac{1}{M}\sum_{i=1}^{M}\sum_{q=1}^{M}\| K(\thetab_2^{(i)}-\thetab_2^{(q)})F(\thetab_1^{(q)})- K(\thetab_2^{(i)}-\thetab_2^{(q)})F(\thetab_2^{(q)})\| \nonumber \\
		&\leq \sum_{i=1}^{M}\beta^{-1}D_F\left\|\thetab_1^{(i)}-\thetab_2^{(i)}\right\|+4D_{\nabla^2K}\sum_{i=1}^{M}\left\|\thetab_1^{(i)}-\thetab_2^{(i)}\right\|+\nonumber \\
		&4\sum_{i=1}^{M}H_FL_{\nabla K}\left\|\thetab_1^{(i)}-\thetab_2^{(i)}\right\|+2\sum_{i=1}^{M}L_FH_{\nabla K}\left\|\thetab_1^{(i)}-\thetab_2^{(i)}\right\|+\nonumber \\
		&2\sum_{i=1}^{M}H_FL_{K}\left\|\thetab_1^{(i)}-\thetab_2^{(i)}\right\|+\sum_{i=1}^{M}L_FH_{K}\left\|\thetab_1^{(i)}-\thetab_2^{(i)}\right\|\nonumber \\
		&\leq(\beta^{-1}D_F+4D_{\nabla^2K}+4H_FL_{\nabla K}
		+2L_FH_{\nabla K}+\nonumber \\
		&2H_FL_{K}+L_FH_{K})\left\|\Thetab_1-\Thetab_2\right\|
		\end{align}     
		
		\item Now, we will prove the last result.
		\begin{align} 
		&\mathbb{E}[\|F^{\Thetab}_j(\Thetab)-\frac{1}{N}\sum \limits_{j=1}^{N}F^{\Thetab}_j(\Thetab)\|^2]=\nonumber \\
		&\sum_{i=1}^{M}\mathbb{E}[\|\beta^{-1}F_j(\thetab^{(i)})-\beta^{-1}\frac{1}{N}\sum \limits_{j=1}^{N}F_j(\thetab^{(i)})\nonumber \\
		&+\frac{1}{M}\sum_{q=1}^{M}K(\thetab^{(i)} - \thetab^{(q)})F_j(\thetab_{t}^{(q)})-\nonumber\\
		&\frac{1}{MN}\sum_{j=1}^{N}\sum_{q=1}^{M}K(\thetab^{(i)} - \thetab^{(q)})F_j(\thetab_{t}^{(q)})\|^2] \nonumber \\
		&\leq \sum_{i=1}^{M} [2\mathbb{E}\|\beta^{-1}F_j(\thetab^{(i)})-\beta^{-1}\frac{1}{N}\sum\limits_{j=1}^{N}F_j(\thetab^{(i)})\|^2\nonumber \\
		&+2\frac{H_K^2}{M^2}\mathbb{E}\|\sum_{q=1}^{M}\left(F_j(\thetab^{(q)}-\frac{1}{N}\sum_{j=1}^{N}F_j(\thetab^{(q)})\right) \|^2] \nonumber \\
		&\leq \sum_{i=1}^{M}({2d \sigma^2+2H_K^2 d \sigma^2})/ N^2\nonumber \\
		&\leq Md(2\sigma^2+2 H_K^2 \sigma^2)/ N^2
		\end{align}
	\end{itemize}
\end{proof}
We apply Euler-Maruyama discretization to Eq.(\ref{eq:extendP}) and substitute $G_k^{\Theta}$ for $F^{\Theta}(\Theta_k)$ to derive the following equation: 
\begin{align}\label{eq:Thedis}
\Theta_{k+1} = \Theta_{k}-G_k^{\Theta}h + \sqrt{2\beta^{-1}h}\Xi_{k},~~\Xi_{k} \sim \mathcal{N}(\mathbf{0}, \Ib_{Md\times Md})\nonumber
\end{align}

Hence, with different $G_k^{\Theta}$, we can perform different algorithm of $\Theta_k$, like SAGA-LD, SVRG-LD and SVRG-$LD^{+}$ algorithm of $\Theta_k$. It is worth noting that the SAGA-LD, SVRG-LD and SVRG-LD$^{+}$ algorithm of $\Theta_k$ is actually the corresponding SAGA-POS, SVRG-POS and SVRG-POS$^+$ algorithm of ${\{\theta_{k}^{(i)}\}}$.\\
This result is extremely important for our proof and bridges the gap between the variance reduction in stochastic gradient Langevin
dynamics (SGLD) and variance reduction in stochastic particle-optimization sampling (SPOS). And thanks to the Theorem \ref{the:them4}, we can can find $F^{\Thetab}(\Thetab)$ satisfies the Assumption \ref{ass:assnew} and Assumption \ref{ass:assnew1} . (Please notice the $F^{\Thetab}(\Thetab)$ corresponds to the $\nabla F$ in \cite{Nilardri:2018}). Hence, we can borrow the theorems in \cite{Nilardri:2018,DIFAN:2018} and derive some thrilling results for the variance reduction techniques in stochastic particle-optimization sampling (SPOS).\\
We denotes the distribution of $\Thetab $ in Eq.(\ref{eq:extendP}) and the distribution of $\Theta_k$ in Eq.(\ref{eq:Thedis}) as $\Gamma_t$ and and $ \Lambda_k$.
Now we can derive the following theorems. ($C_1$,$C_2$,$C_3$,$C_4$ and $C_5$ are defined in Section 4)
\begin{theorem}\label{thm:The1}
	Let the step size $h < \frac{B}{8NC_1}$ and the batch size $B \geq 9$, then we can have the bound for $\mathcal{W}_2(\Lambda_T,\Gamma_{\infty})$ in the SAGA-LD algorithm of $\Theta_k$.
	
	\begin{align}
	\mathcal{W}_2(\Lambda_T,\Gamma_{\infty})\leq &
	5 \exp(-\frac{C_3 h}{4}T)\mathcal{W}_2(\Lambda_0,\Gamma_{\infty})+ \nonumber\\ 
	&\frac{2hC_4Md}{C_3}+\frac{2h{C_2}^\frac{3}{2}\sqrt{Md}}{C_3}+\frac{24hC_2\sqrt{MdN}}{\sqrt{C_3}B} \nonumber
	\end{align}
\end{theorem}

\begin{theorem} \label{thm:The2}
	If we choose Option $\uppercase\expandafter{\romannumeral1}$ and set the step size $h < \frac{1}{8C_2}$, the batch size $B \geq 2$ and the epoch length $\tau \geq \frac{8}{C_3 h}$, then we can have the bound for all T mod $\tau$ =0 in the SVRG-LD algorithm of $\Theta_k$. 
	\begin{align}
	\mathcal{W}_2(\Lambda_T,\Gamma_{\infty})\leq & \exp(-\frac{C_3h}{56}T)\frac{\sqrt{C_2}}{\sqrt{C_3}} \mathcal{W}_2(\Lambda_0,\Gamma_{\infty})+\nonumber \\
	&\frac{2hC_4Md}{C_3}+\frac{2h{C_2}^\frac{3}{2}\sqrt{Md}}{C_3}+\frac{64C_2^{\frac{3}{2}}\sqrt{hMd}}{\sqrt{B}C_3}\nonumber
	\end{align}
	If we choose Option $\uppercase\expandafter{\romannumeral2}$ and set the step size $h < \frac{\sqrt{B}}{4\tau C_2}$, then we can have the bound for all T in the SVRG-LD algorithm of $\Theta_k$. 
	\begin{align}
	\mathcal{W}_2(\Lambda_T,\Gamma_{\infty})\leq &\exp(-\frac{C_3h}{4}T)\mathcal{W}_2(\Lambda_0,\Gamma_{\infty})+\nonumber\\
	&\frac{\sqrt{2}hC_4Md}{C_3}+\frac{5h{C_2}^\frac{3}{2}\sqrt{Md}}{C_3}+\frac{9hC_2\tau\sqrt{Md}}{\sqrt{BC_3}}\nonumber
	\end{align}
\end{theorem}
\begin{theorem} \label{thm:The3}
	If we set the step size $h \leq min\{(\frac{BC_3 }{24{C_2}^4\tau^2})^{\frac{1}{3}},\frac{1}{6\tau({C_5}^2/b+C_2)}\}$, then we can have the bound for all T in the algorithm SVRG-LD$^+$ of $\Theta_k$. 
	\begin{align}
	&\mathcal{W}_2(\Lambda_T,\Gamma_{\infty})\leq (1-hC_2/4)^T\mathcal{W}_2(\mu_0,\mu^*)+\nonumber\\
	&\frac{3C_5 (Md)^{1/2}}{C_3b^{1/2}}\boldsymbol{1}(b\leq N)+\frac{2h(C_4Md)}{C_3}+\frac{2h{C_2}^{3/2}(Md)^{1/2}}{C_3}\nonumber\\
	&+\frac{4hC_2(\tau Md)^{1/2}\wedge 3h^{1/2}(Md)^{1/2}C_5}{\sqrt{BC_3}}\nonumber
	\end{align}
\end{theorem}
Now we will give a proposition which will be useful in connecting the $\mathcal{W}_2(\Lambda_T,\Gamma_{\infty})$ and $\mathcal{W}_2(\mu_{T},\nu_{\infty})$ mentioned above.

\begin{proposition}\label{pro:pro1}
	(For simplicity of notations, we directly use $\theta$  and $\Theta$ themselves to denote their own distributions.) If $\Thetab_1$ and $\Thetab_2$ are defined as $\Thetab_1\triangleq [\thetab_{1}^{(1)}, \cdots, \thetab_{1}^{(M)}] \in \mathbb{R}^{Md}$
	and $\Thetab_2\triangleq [\thetab_{2}^{(1)}, \cdots, \thetab_{2}^{(M)}] \in \mathbb{R}^{Md}$,
	we can derive the following result
	\begin{align}
	\sum_{i=1}^{M}\mathcal{W}^2 _2(\thetab_{1}^{(i)},\thetab_{2}^{(i)})\leq \mathcal{W}^2 _2(\Thetab_1,\Thetab_2)
	\end{align}
\end{proposition}
\begin{proof}
	According to the Eq.(4.2) in \cite{Soheilal2017}, we can write the $\mathcal{W}_2(\thetab_{1}^{(i)},\thetab_{2}^{(i)})$ in the following optimizaition:
	\begin{align} \label{eq:w2}
	\mathcal{W}^2 _2(\thetab_{1}^{(i)},\thetab_{2}^{(i)})=\mathbb{E}\|\thetab_{1}^{(i)}\|^2+\mathbb{E}\|\thetab_{2}^{(i)}\|^2\nonumber\\
	+2\sup_{\phi:convex} \{-\mathbb{E}[\phi(\thetab_{1}^{(i)})]-\mathbb{E}[\phi^*(\thetab_{2}^{(i)})]\}
	\end{align}
	where $\phi^*(\theta) \triangleq \sup_{v}(v^{T}\theta-\phi( \theta))$ is the convex-conjugate of the function $\phi$.
	We assume $\phi_i$ is the optimal function of Eq.\ref{eq:w2}.
	Then it is trivial to verify that $\Psi(\Thetab)\triangleq \sum_{i=1}^{M}\phi_i(\thetab^{(i)})$ is a convex function. Due to the property of conjugate functions, we need to notice $\Psi(\Thetab)^*= \sum_{i=1}^{M}\phi^{*}_i(\thetab^{(i)})$.
	Now we can derive the following result:
	\begin{align*}
	\sum_{i=1}^M\mathcal{W}^2 _2(\thetab_{1}^{(i)},\thetab_{2}^{(i)})=&
	\sum_{i=1}^M \{\mathbb{E}\|\thetab_{1}^{(i)}\|^2+\mathbb{E}\|\thetab_{2}^{(i)}\|^2\nonumber\\
	&+2(-\mathbb{E}[\phi_i(\thetab_{1}^{(i)})]-\mathbb{E}[\phi^*_{i}(\thetab_{2}^{(i)})])\}  \\
	&= \mathbb{E}\|\Thetab_1\|^2+\mathbb{E}\|\Thetab_2\|^2\\
	&+2(-\mathbb{E}[\Psi(\Thetab_1)]-\mathbb{E}[\Psi^*(\Thetab_2)])\\
	&\leq \mathcal{W}^2 _2(\Thetab_1,\Thetab_2)
	\end{align*}
	Then we finish our proof.
\end{proof}

%
%

We should notice due to the exchangeability of the M-particles system $\{\theta^{(i)}_k\}$ in our SPOS-type sampling, the distribution of each particle $\theta^{(i)}_T$ at the same time is identical. Hence, using Proposition \ref{pro:pro1}, we can derive
\begin{align}
\mathcal{W}_2(\mu_{T},\nu_{\infty}) \leq \frac{1}{\sqrt{M}}\mathcal{W}_2(\Lambda_T,\Gamma_{\infty})
\end{align}\\

Now we will introduce a mild assumption that $\mathcal{W}_2(\mu_{T},\nu_{\infty}) \leq \frac{1}{{M}^{1/2+\alpha}}\mathcal{W}_2(\Lambda_T,\Gamma_{\infty})$. We wish to make some comments on the additional assumption. This assumption is reasonable. With this assumption, our theory can be verified by the experiment results, e.g. the improvement of SVRG-POS over SVRG-LD is much more significant than that of SAGA-POS over SAGA-LD, which imply the correctness and effectiveness of our assumption. Moreover, this assumption does not conflict with what you mentioned, since $\mathcal{W}_2(\mu_{T},\nu_{\infty}) \leq \frac{1}{{M}^{1/(2+\alpha)}}\mathcal{W}_2(\Lambda_T,\Gamma_{\infty}) \leq \frac{1}{\sqrt{M}}\mathcal{W}_2(\Lambda_T,\Gamma_{\infty})$. Furthermore, this assumption can be supported theoretically. Please consider the continuous function $\log_M \left( W_2(\Thetab_1,\Thetab_2) M/ \sum_{i=1}^M W_2(\thetab^{(i)}_1, \thetab^{(i)}_2) \right)-1/2$. We often care about bounded space in practice, which means we can find a positive minimum for that function in most cases. Since in practice we cannot use infinite particles, the required $\alpha$ does exist within the positive minima for every M mentioned above. Although we do not aim at giving an explicit expression for it, the existence is enough to explain the experiment results in our paper. Last, this assumption is supported in the algorithm itself. Please notice the fact that SPOS can be viewed as the combination of SVGD and SGLD. The SVGD part can let it satisfy some good properties which SGLD does not endow.\\

\textbf{Proof of Theorem \ref{thm:thm1}, Theorem \ref{thm:thm2} and Theorem \ref{thm:thm3}}
Applying the results for $\mathcal{W}_2(\Lambda_T,\Gamma_{\infty})$ in Theorem \ref{thm:The1}, Theorem \ref{thm:The2} and Theorem \ref{thm:The3}, we can get the corresponding results for $\mathcal{W}_2(\mu_{T},\nu_{\infty})$ in the SAGA-POS, SVRG-POS and SVRG-POS$^+$.
Then we can bound $\mathcal{W}_2(\mu_{T},\mu^*)$,which is what we desire, with the following fact
\begin{align}
\mathcal{W}_2(\mu_{T},\mu^*) \leq \mathcal{W}_2(\mu_{T},\nu_{\infty}) +\mathcal{W}_2(\nu_{\infty},\mu^*)
\end{align}
Note that from the proof of Theorem 3 and Remark 1 in  \cite{ZhangZC:18}, we can get that  
\begin{align}
\mathcal{W}_2(\nu_{\infty},\mu^*)\leq \frac{C_1}{\sqrt{M}}
\end{align}
Apply the results in Theorem \ref{thm:The1}, Theorem \ref{thm:The2} and Theorem \ref{thm:The3} above, we can prove the Theorem \ref{thm:thm1}, Theorem \ref{thm:thm2} and Theorem \ref{thm:thm3}.

\section{Extra theoretical discussion for SAGA-POS, SVRG-POS and SVRG-POS$^+$}
In this section, we discuss the mixing time and gradient complexity of our algorithms. The mixing time is the number of iterations needed to provably have error less than $\varepsilon$ measured in $\mathcal{W}_2$ distance \cite{Nilardri:2018}. The gradient complexity \cite{DIFAN:2018}, which is almost same as computational complexity in \cite{Nilardri:2018}, is defined as the required number of stochastic gradient evaluations to achieve a target accuracy $\varepsilon$.
We will present the mixing time and gradient complexity of several related algorithms in the following Table \ref{tab:mixt1}. And we focus on Option \uppercase\expandafter{\romannumeral1} of SVRG-POS here. This result for SVRG-LD$^+$ and SVRG-POS$^+$ may be a little different from that in \cite{DIFAN:2018} since we adopt different definitions for $F_j$.\\
\begin{table}[htbp]
	\vspace{-0.5cm}
	\centering
	\caption{Mixing Time and Gradient Complexity}
	\resizebox{1\linewidth}{30mm}{
		\begin{tabular}{ccc}
			\toprule[0.8pt]
			Algorithm & Mixing time  &  Gradient complexity \\ \hline
			SAGA-LD &      $\mathcal{O}(\frac{(L_F/m_F)^{3/2}\sqrt{d}}{B\varepsilon})$          &   $\mathcal{O}(N+\frac{(L_F/m_F)^{3/2}\sqrt{d}}{\varepsilon})$                   \\ 
			\\
			SAGA-POS &     $\mathcal{O}(\frac{(C_2/C_3)^{3/2}\sqrt{d}}{BM^{\alpha}\varepsilon })$  &  $\mathcal{O}(NM+\frac{(C_2/C_3)^{3/2}\sqrt{d}M^{1-\alpha}}{ \varepsilon })$                    \\
			\\
			SVRG-LD  &       $\mathcal{O}(\frac{(L_F/m_F)^{3}d}{ B \varepsilon^2 })$         &$\mathcal{O}(N+\frac{(L_F/m_F)^{3}\sqrt{d}}{ \varepsilon^2})$                     \\
			\\
			SVRG-POS &        $\mathcal{O}(\frac{(C_2/C_3)^{3}d}{ BM^{2\alpha} \varepsilon^2 })$      &$\mathcal{O}(NM+\frac{(C_2/C_3) ^{3}\sqrt{d}M^{1-2\alpha}}{ \varepsilon^2})$                     \\
			\\
			SVRG-LD$^+$ &     $\mathcal{O}(\frac{\sigma^2 d}{ {m_F}^2 \varepsilon^2 })$      & $\mathcal{O}(\frac{\sigma^2 d}{ {m_F}^2 \varepsilon^2 }\wedge (N+\frac{(L_F/m_F)^{3/2}\sqrt{d}}{\varepsilon}))$                    \\
			\\
			SVRG-POS$^+$ &     $\mathcal{O}(\frac{C_4^2 d}{ M^{2\alpha}{C_3}^2 \varepsilon^2 })$      & $\mathcal{O}(\frac{C_4^2 d}{ {C_3}^2 \varepsilon^2 }\wedge (NM+\frac{(C_2/C_3)^{3/2}\sqrt{d}M^{1-2\alpha}}{\varepsilon}))$                     \\
			\bottomrule[0.8pt]
		\end{tabular}
	}
	\label{tab:mixt1}
\end{table}

It is worth noting that the results for SVRG-POS$^+$ is derived by adopting that $B=1$ and $b= \mathcal{O}({d\sigma^2/ \mu^2 {\varepsilon}^2 })$ from \cite{DIFAN:2018}, which also sheds a light on the optimal choice of b and B in our SVRG-POS$^+$. For fair comparisons with our algorithms, we consider variance-reduced versions of SGLD with M independent chains. Hence, the gradient complexities of the SAGA-LD, SVRG-LD and SVRG-LD$^+$ need to times M respectively, which is consistent with the discussion in Section 3 and our experiment results. As for the case there is only one chain, we can find the  gradient complexities of SVRG-LD and SVRG-POS are almost the same and the gradient complexities of SVRG-LD$^+$ and SVRG-POS$^+$ are almost the same. You can choose  SVRG-POS and SVRG-POS$^+$ for practical use to improve your results. However, we need to note that in practice, we do need more chains to  derive samples which are more convincible. So this also provide another reason for us to use more chains to compare the results. \\
Since the convergence guarantee in Theorem \ref{thm:thm1}, \ref{thm:thm2} and \ref{thm:thm3} for our  is developed with respect to both iteration $T$ and the number $M$, so we define the "threshold-particle", which means the number of particles needed to provably have error less than $\varepsilon$ measured in $\mathcal{W}_2$ distance. We will present the "threshold-particle" for our algorithms.

\begin{table}[htbp]
	\vspace{-0.5cm}
	\centering
	\caption{Threshold-particle}
	\begin{tabular}{cc}
		\toprule[0.8pt]
		Algorithm & Threshold-particle \\ \hline
		SAGA-POS &     $C_1^2/\varepsilon^2$ 
		\\
		SVRG-POS &        $C_1^2/\varepsilon^2$     
		\\
		SVRG-POS$^+$ &     $C_1^2/\varepsilon^2$ 
		\\      
		\bottomrule[0.8pt]
	\end{tabular}
	\label{tab:mixt2}
\end{table}
Actually, we should note that the M in the mixing time from Table \ref{tab:mixt1} also should satisfy the result that $M \geq C_1^2 / \varepsilon^2$. In practice, since $C_1=\frac{H_{\nabla K}+H_F}{\sqrt{2}(\beta^{-1}-3H_FL_K-2L_F)}$, we set $\beta$ a little small to avoid the threshold-particle to be too large. However, in our experiment since $H_F$,$L_K$ and $L_F$ is not large, so we do not need to worry about this issues.\\
The last thing we want to give some explanations is that one may notice in the SAGA-POS algorithm we may store even more things like $G_k^{(i)}$. Since in Algorithm 1, we need to store them each iteration. But the $\{G_k^{(i)}\}_{i=1}^M$ only scales as $\mathcal{O}(Md)$. Take the dataset $Susy$, which we used in the experiments, as an example, we have $N=10000$ and $d=18$. However, we only use $M\leq 40$ particles. Hence, we do not take the above things such as $\{G_k^{(i)}\}_{i=1}^M$ into consideration.

\section{Comparison between SPOS and its variance-reduction counterpart}
In \cite{ZhangZC:18}, they use the distance $\tilde{\mathcal{B}}_T$ defined as  $\tilde{\mathcal{B}}_T \triangleq \sup \left|\mathbb{E}_ {\mu_{T}}[f(\thetab)] - \mathbb{E}_{\mu^*}[f(\thetab)]\right|$. When $\|f\|_{lip}\leq 1$, $\tilde{\mathcal{B}}_T$ is another definition of $\mathcal{W}_1(\mu_{T},\mu^*)$.Actually, according the proof in \cite{ZhangZC:18}, they did give a bound in terms of $\mathcal{W}_1(\mu_{T},\mu^*)$. Then according to the results in \cite{ZhangZC:18}, we can get the following theorem,
\begin{theorem}[Fixed Stepsize]\label{theo:fixed}
	Under Assume~\ref{ass:ass1}, there exit some positive constants $(c_1, c_2, c_3, c_4,c_5)$ such that the bound for $\mathcal{W}_1(\mu_T,\mu^*)$ in the SPOS algorithm satisfies:
	\begin{align*}
	&\mathcal{W}_1(\mu_{T},\mu^*)\leq \frac{c_1}{\sqrt{M}(\beta^{-1} - c_2)}\\
	& + W_1(\mu_{0}, \mu_{\infty})\exp\left\{-2\left(\beta^{-1}m_F - 3H_FL_K - 2L_F\right)Th\right\} \\
	&+ c_3Md^{\frac{3}{2}}\beta^{-3}(c_4\beta^2B^{-1}+c_5h)^{\frac{1}{2}} T^{\frac{1}{2}}h^{\frac{1}{2}}~.
	\end{align*}
\end{theorem}
Firstly, we should notice that the third term $c_3Md^{\frac{3}{2}}\beta^{-3}(c_4\beta^2B^{-1}+c_5h)^{\frac{1}{2}} T^{\frac{1}{2}}h^{\frac{1}{2}}$ on the right side increases with $T$ and $M$. However, the bound for SAGA-POS, SVRG-POS and SVRG-POS$^+$ in our paper decrease with with $T$ and $M$, which means that the bound for SAGA-POS, SVRG-POS and SVRG-POS$^+$ is much tighter than the bound for SPOS. Furthermore, the convergence of  SPOS is characterized in  $\mathcal{W}_1(\mu_T,\mu^*)$ but the convergence of SAGA-POS, SVRG-POS and SVRG-POS$^+$ are characterized by $\mathcal{W}_2(\mu_{T},\mu^*)$. Due to the well-known fact that $\mathcal{W}_1(\mu_T,\mu^*) \leq \mathcal{W}_1(\mu_T,\mu^*)$, we can verify that SAGA-POS, SVRG-POS and SVRG-POS$^+$ can outperform SPOS in the theoretical perspective. Although the result for SPOS in \cite{ZhangZC:18} may be improved in the future, but we believe that there is no doubt that SAGA-POS, SVRG-POS and SVRG-POS$^+$ are better than it in performance, which has been verified in experiments in our paper.

\section{More experiments results}

We further examine the impact of number of particles to the convergence rates of variance-reduced SGLD and SPOS. As indicated by Theorems~\ref{thm:thm1}-\ref{thm:thm3} (discussed in Remark~\ref{remark1} and \ref{remark2}), when the number of particles are large enough, the convergence rates of SAGA-POS and SVRG-POS would both outperform their SGLD counterparts. In addition, the performance gap would increase with increasing $M$, as indicated in Remark~\ref{remark4}. We conduct experiments on the $Australian$ dataset by varying its particle numbers among $\{1, 8, 16, 32\}$. The results are plotted in Figure~\ref{fig:part}, which are roughly aligned with our theory.
\begin{figure}
	\centering
	\subfigure[1 particle]{\includegraphics[width=0.9\linewidth]{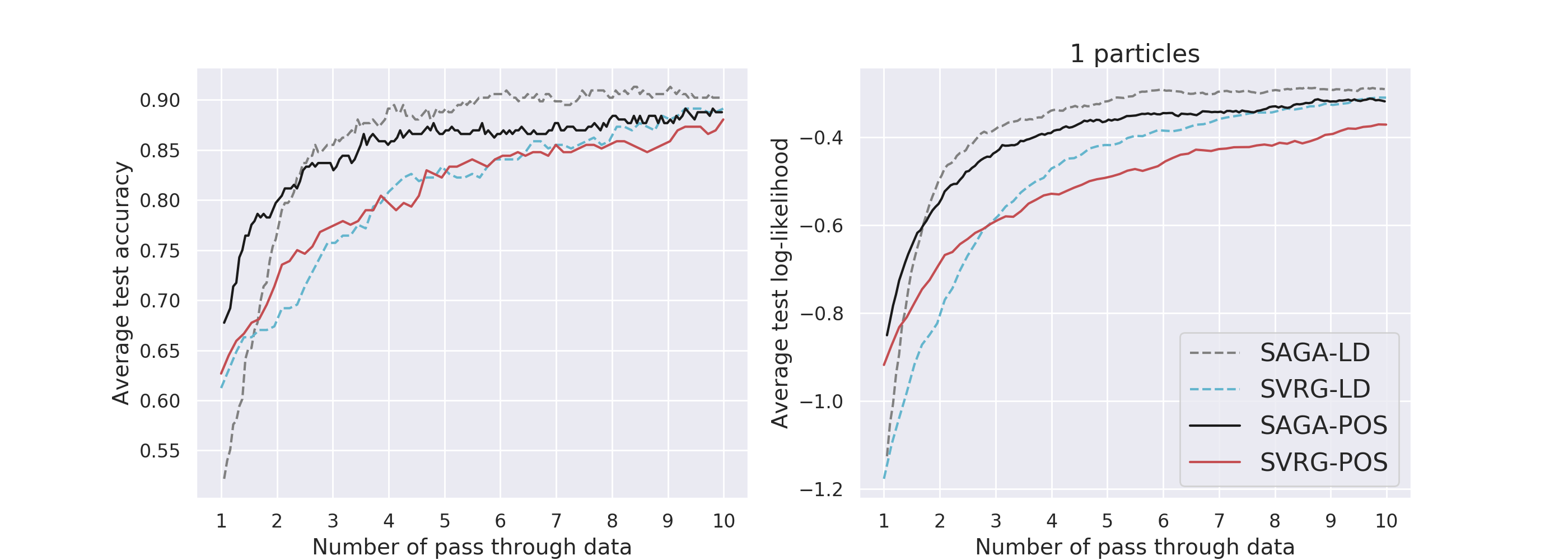}}\\
	\vspace{-0.4cm}
	\subfigure[8 particles]{\includegraphics[width=0.9\linewidth]{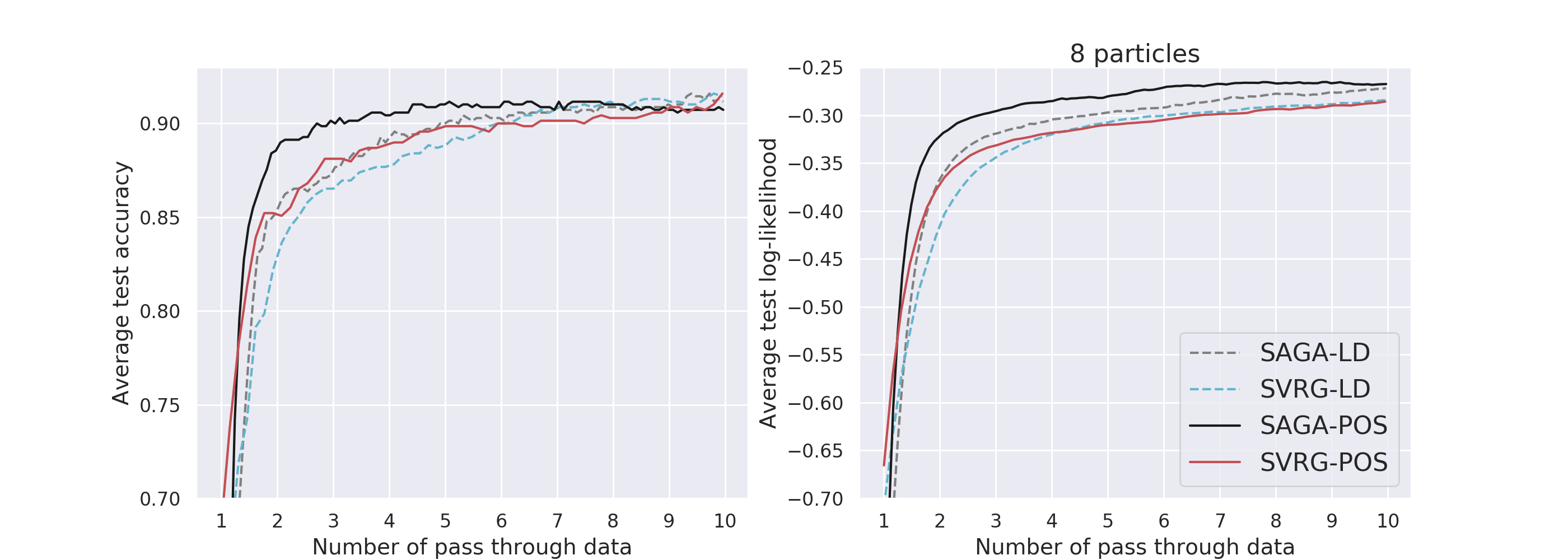}}
	\vspace{-0.1cm}
	\subfigure[16 particles]{\includegraphics[width=0.9\linewidth]{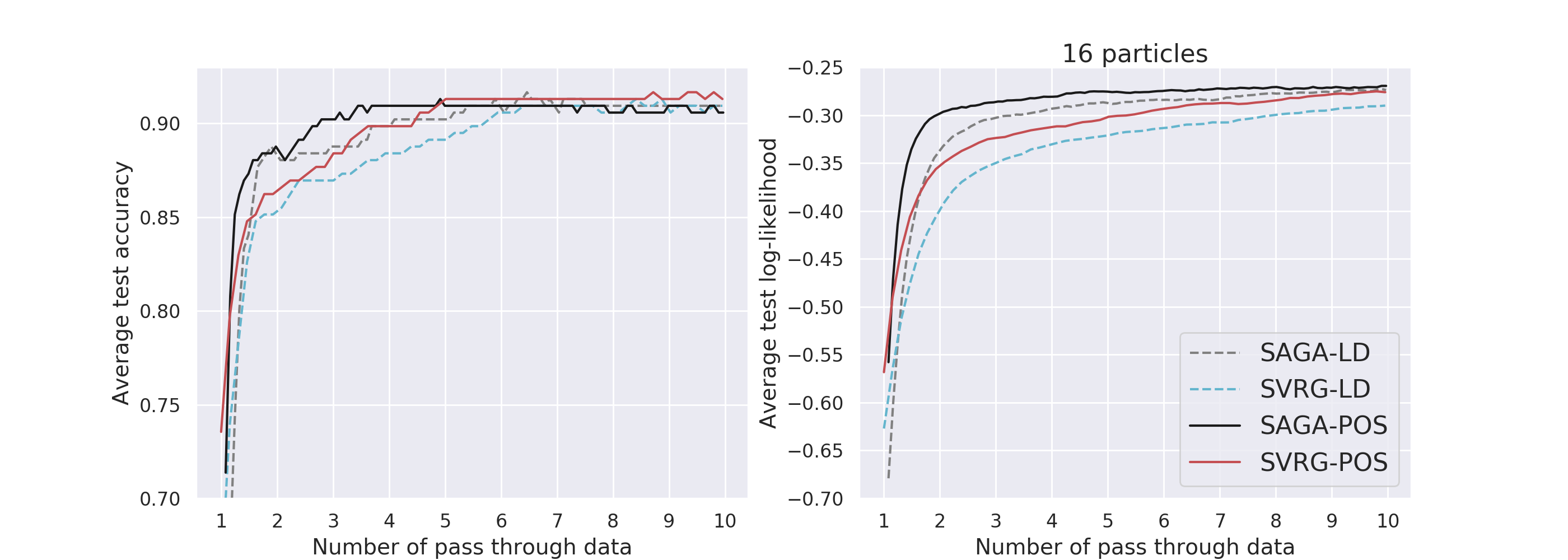}}\\
	\vspace{-0.4cm}
	\subfigure[32 particles]{\includegraphics[width=0.9\linewidth]{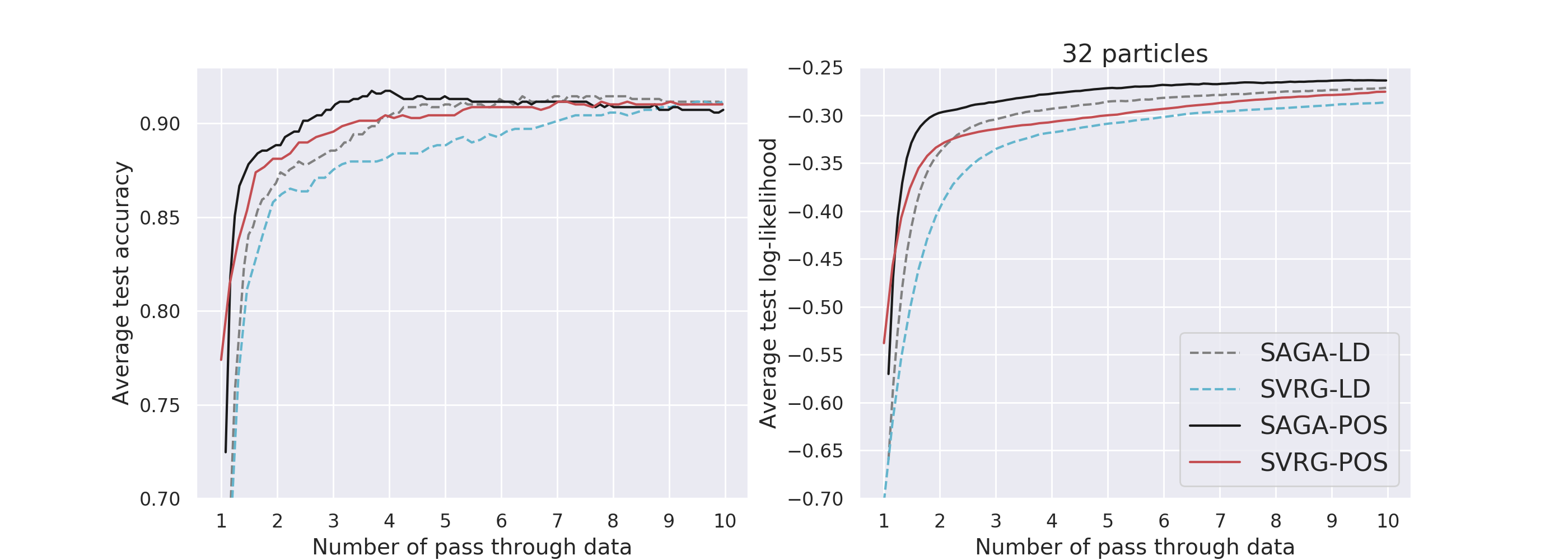}}		
	\vspace{-0.3cm}
	\caption{Testing accuracy and log-likelihood vs the number of data pass for SPOS with varying number of particles.}
	\label{fig:part}
\end{figure}

\end{document}